\documentclass[journal]{IEEEtran}
\pdfoutput=1
\usepackage{graphicx}
\usepackage{amsmath}
\usepackage{latexsym}
\usepackage{amssymb}
\usepackage{subfigure}
\usepackage{hyperref}
\usepackage{cite}
\usepackage{algorithm}
\usepackage{algorithmic}
\usepackage{multirow}
\usepackage{color}
\newtheorem{theorem}{Theorem}

\begin{document}
\title{Discriminative Transformation Learning for Fuzzy Sparse Subspace Clustering}
\author{Zaidao Wen,~Biao Hou,~\emph{Member,~IEEE},~Qian Wu and Licheng Jiao,~\emph{Senior Member,~IEEE} % <-this % stops a space
}
\maketitle

\begin{abstract}
This paper develops a novel iterative framework for subspace clustering in a learned discriminative feature domain. This framework consists of two modules of fuzzy sparse subspace clustering and discriminative transformation learning. In the first module, fuzzy latent labels containing discriminative information and latent representations capturing the subspace structure will be simultaneously evaluated in a feature domain. Then the linear transforming operator with respect to the feature domain will be successively updated in the second module with the advantages of more discrimination, subspace structure preservation and robustness to outliers. These two modules will be alternatively carried out and both theoretical analysis and empirical evaluations will demonstrate its effectiveness and superiorities. In particular, experimental results on three benchmark databases for subspace clustering clearly illustrate that the proposed framework can achieve significant improvements than other state-of-the-art approaches in terms of clustering accuracy.

\end{abstract}


% Will be ignored for other modes.
\IEEEpeerreviewmaketitle

\section{Introduction}
\label{sec:intro}
\PARstart{S}{ubspace} segmentation is one of the fundamental tasks in computer vision and image processing whose objective is to recover the intrinsic subspaces, their numbers as well as dimensions from a collection of observed samples embedding in an ambient high dimensional space, and then assign each sample to the correct subspace. To implement this task, the core issue will become to cluster these samples into groups based on their underlying subspace memberships. Numerous methods have been devoted to the task of subspace clustering (SC) in recent decades, including the iterative clustering frameworks \cite{Tseng2000}, statistical methods \cite{Tipping1999}, algebraic approaches \cite{Kanatani2001Motion}, spectral clustering based frameworks \cite{Yan2006,Chen2009}, \emph{etc} and they have been successfully adapted to various practical applications such as motion segmentation, handwritten image clustering as well as face clustering.
\par Among the above mentioned methods, spectral clustering based framework becomes much more appealing due to its simplicity, in which only a suitable affinity matrix will be critically constructed to reflect the pairwise relationships among the observed samples \cite{Luxburg2007}.  Exploiting the property of linear subspace, self-expressiveness or autoregressive model can provide an effective way of building an affinity matrix for SC. It suggests that any a sample can be represented by the linear combinations of some other ones if they are lying in the subspace \cite{Elhamifar2013}. It follows that the overall latent self-representation coefficients with respect to a set of samples will capture their structured relationships so as to construct an affinity matrix. In this spirit, spectral clustering based SC becomes to compute the self-representation coefficients via solving a linear inverse problem, where a regularization can be imposed on coefficients to obtain the informative and stable solutions. Elhamifar and Vidal developed a sparse subspace clustering (SSC) framework by involving a sparsity-inducing norm for regularization \cite{Elhamifar2013}, which significantly improves the clustering performances in many applications than other approaches\cite{Tron2007}. Following this baseline framework, many subsequent researches equipped with various regularizations are gradually developed for SC \cite{VidalSubspaceclstering}, including $\ell_1$ norm \cite{Soltanolkotabi2014,Patel2013}, $\ell_2$ norm \cite{Peng2015}, Frobenius norm \cite{Lu2012}, nuclear norm \cite{Chen2014,Kang2015,Liu2013,Fang2016,Wang2017}, the combinations  \cite{Wang2013,You2016Oracle} and some structure inducing variants \cite{Li2015,Li2016,Feng2014Robust,Peng2016}, \emph{etc}.
\subsection{Motivation and Related Work}
\par According to the previous introduction, the above spectral clustering based frameworks for SC will rely on two separate phases. In the first phase, given a set of high dimensional observations, an affinity matrix will be constructed from the solution of a regularized linear inverse problem. Then we carry out the standard spectral clustering algorithm to obtain the clustering labels for these samples. However, such paradigm generally suffers from the following deficiencies. Firstly, determining the clustering labels or memberships based on the resulted representations will be fundamentally sub-optimal due to ignoring their underlying dependencies \cite{Li2016,Li2015}. In other words, the membership will essentially influence the discrimination of the representations, which is not taken into account in the aforementioned frameworks. Secondly, the intrinsic subspace structure enabling clustering will be generally violated in the high dimensional observation domain due to some realistic contaminations \cite{Qiu2014}. As a consequence, the observed samples will be generally entangled and biased from their intrinsic subspaces which aggravates the difficulties and computational complexities for clustering in that domain.
\par To address the first problem, Li and Vidal presented a unified framework for learning the affinity matrix and clustering labels simultaneously to capture their dependencies \cite{Li2015}, and they recently extend this framework to the task of subspace completion and clustering to address the problem of missing entries in observed samples \cite{Li2016}. However, they still perform SC in the primary domain without further involving the discriminative strategies to disentangle the high dimensional observations. Considering the issue of high dimensionality, a common strategy in many SC frameworks is firstly to project the samples onto a low dimensional subspace with an off-the-shelf dimensionality reduction (DR) method, \emph{e.g.}, principal component analysis (PCA) \cite{Smith2002} and random projection \cite{Bingham2001} and then carry out the clustering algorithm in this subspace \cite{Elhamifar2013}\cite{Liu2013}. Although this architecture has been widely exploited in practical situations, these DL algorithms do not concern the subspace structures among samples, which will result in a sacrificed clustering accuracy in general. To address this issue, Patel \emph{et al.} proposed a latent space sparse subspace clustering (LS3C) algorithm for simultaneously DR and SC in a united framework \cite{Patel2013}, which is further extended to a kernelized nonlinear framework (NLS3C) \cite{Patel2014Kernel,Patel2015Latent,Hofmann2008}. In their frameworks, the linear DR operator can preserve the subspace structure by considering the self-representations, but they still ignore the subspace memberships during computing the DR operator and representations. It follows that samples in the resulted latent space do not exhibit more separability to improve the clustering accuracy remarkably. To find a more discriminative latent space, Qiu and Sapiro learned a discriminative low rank transformation (LRT) by involving the estimated labels, which can encourage a maximally separated structure for inter-class subspaces and reduce the variation within the intra-class subspaces \cite{Qiu2014}. In their research, they combined LRT with a robust SSC deriving from local linear embedding (LLE) \cite{Roweis2000} for SC, which achieves significant improvements on face clustering and classification than other algorithms. However, their framework still suffers from many deficiencies, such as less robustness to outlying labels, getting stuck in trivial or pathological solutions, \emph{etc.}, which will increase the potential risk of performance collapsing.
\subsection{Main Contributions}
In this paper, we will propose a novel framework for subspace clustering in a discriminative feature subspace to address the above problems. The main contributions are summarized as following.
\begin{itemize}
  \item We develop a novel fuzzy sparse subspace clustering approach in which the self-representations and a fuzzy label matrix will be collaboratively computed.
  \item Based on the fuzzy label matrix and self-representations, a linear discriminative transformation operator will be learned, by which the subspace structure can be preserved in the feature subspace while the discrimination and robustness can be significantly improved.
  \item We propose an effective optimization scheme to solve the above two procedures alternatively, which can effectively prevent from the local optimal or trivial solutions.
\end{itemize}
The experimental results on three typical benchmarks for the applications of motion segmentation, digital handwritten clustering, face clustering all demonstrate the effectiveness of our proposed framework. In particular, the superiority of our framework is validated by sharply improving the clustering accuracy by a large margin, compared with the other state-of-the-art subspace clustering approaches.
\par The rest paper is organised as follows. In Sec. \ref{Sec:Prelimiinary}, we make a detailed review of related works to highlight our motivations. Sec. \ref{Sec:RSSSC} proposes our framework in detail and we derive the optimization scheme in Sec. \ref{Sec:OPT}. Extensive experiments are conducted in Sec. \ref{Sec:Experiments} to demonstrate its effectiveness and superiorities and Sec. \ref{sec:Conclusion} concludes this paper.

\section{Preliminaries}\label{Sec:Prelimiinary}
In this section, we will formally formulate the problem of SC from a specific probabilistic insight to highlight our motivations and simultaneously review some related works.
\subsection{Regularization Based Subspace Clustering}
Let $\mathbf{X}=\left[\mathbf{x}_1,\dots,\mathbf{x}_N\right]\in\mathbb{R}^{n\times N}$ be a collection of $n$-dimensional observations drawn from a union of $K$ independent or disjoint linear subspaces $\{\mathcal{S}_k\}_{k=1}^K$. The task of subspace clustering is labelling all observations such that those from the same subspace will have the same label. Due to lack of supervised information, the general clustering algorithms will mostly attempt to capture the underlying data structures or relationships by exploiting some prior assumptions to implement this task, \emph{e.g.}, local label consistency, manifold assumption, \emph{etc}. Essentially, these algorithms attempt to character data distribution $p(\mathbf{X})=\sum_{\mathcal{Z}}p(\mathbf{X}|\mathcal{Z})p(\mathcal{Z})$ with some unsupervised generative models, where $\mathcal{Z}$ contains the latent variables generating the data $\mathbf{X}$. Then these latent variables with various structural priors will be exploited for subsequent clustering task. In practical scenario, since this marginal distribution will be always difficult or prohibitive to access straightforwardly, some approximating strategies have to be leveraged instead by maximizing its lower bound as the left hand side of the following equation, yielding a maximum a posterior (MAP) estimation of $\mathbf{Z}$.
\begin{equation}\label{Equ:LowboundData}
 \max_{\mathbf{Z}} p(\mathbf{X|Z})p(\mathbf{Z}) \leq \sum_{\mathcal{Z}}p(\mathbf{X}|\mathcal{Z})p(\mathcal{Z})=p(\mathbf{X})
\end{equation}
Considering SC, the property of self-expressiveness or auto-regression for linear subspace becomes much appealing in recent years to model data likelihood $p(\mathbf{X|Z})$ \cite{Wright2009}, which suggests that a data point in a union of subspaces can be linearly approximated by other points \cite{Elhamifar2013}. More substantial, any sample $\mathbf{x}_i$ in above $\mathbf{X}$ will be formulated as
\begin{equation}\label{Equ:Selfexpressive}
  \mathbf{x}_i= \mathbf{X}_{\setminus i}\mathbf{z}_i+\mathbf{e}_i=\sum_{j\neq i}\mathbf{x}_j\mathbf{z}_{i}(j)+\mathbf{e}_i
\end{equation}
where $\mathbf{X}_{\setminus i}$ is the sample matrix excluding $i$-th column, $\mathbf{e}_i$ models the residual vector and $\mathbf{z}_{i}(j)$ is the $j$-th entry in the so-called self-representation coefficient vector $\mathbf{z}_i$. According to this expressive model, MAP in \eqref{Equ:LowboundData} can be generally reformulated as the following regularized minimization form.
\begin{equation}\label{Equ:RegularizedSC}
  \{\mathbf{z}_i^*\}_{i=1}^N\in\arg\min_{\mathbf{z}_i} \sum_{i=1}^N\ell_{\mathcal{X}}(\mathbf{x}_i,\mathbf{X}_{\setminus i}\mathbf{z}_i)+\alpha J(\mathbf{z}_1,\dots,\mathbf{z}_N)
\end{equation}
where $J(\cdot)$ corresponds to a prior inducing regularization function on $\mathbf{Z}$ with hyper-parameter $\alpha$ and $\ell_{\mathcal{X}}(\cdot)$ is a metric of the input domain $\mathcal{X}$ induced from the data likelihood that generally characterizes the representation error $\mathbf{e}_i$. Once the MAP solutions $\{\mathbf{z}_i\}_{i=1}^N$ are computed, structural affinities between data point $\mathbf{x}_i$ and other points can be globally encoded in $\mathbf{z}_i$, \emph{i.e.}, the magnitude $\left|\mathbf{z}_i(j)\right|$ can be regarded as the contribution of $\mathbf{x}_j$ in generating $\mathbf{x}_i$. Since this relationship is generally not symmetric as $\mathbf{z}_i(j)\neq \mathbf{z}_{j}(i)$, an affinity between $\mathbf{x}_i$ and $\mathbf{x}_j$ can be designed as $|\mathbf{z}_i(j)|+|\mathbf{z}_j(i)|$ to make it symmetric. Clustering can be subsequently implemented via normal spectral approaches based on the constructed affinity matrix, such as spectral clustering \cite{Luxburg2007}, normalized cut \cite{Shi2000}.
\par In the spirit of  \eqref{Equ:RegularizedSC}, various researches focus on elaborating different regularizations $J(\cdot)$ for capturing different structural properties of the subspaces, which have achieved significant progresses in many practical applications. In addition to $J(\cdot)$, they also attempt to select different loss functions $\ell_{\mathcal{X}}(\mathbf{x}_i,\mathbf{X}_{\setminus i}\mathbf{z}_i)$ to character different conditional data distributions for the sake of enhancing the robustness of model \cite{Chen2014,Elhamifar2013,Liu2013}. It is, however, of complexity and difficulty to choose an appropriate metric in the high dimensional domain \cite{Davis2008} so that some other variants gradually turn to searching a feature domain to overcome these shortages \cite{Patel2013,Patel2014Kernel,Hou2015}. Instead of exploiting an off-the-shelf non-parametric DR operator, a parametric transformation operator from $\mathcal{X}$ to a latent $p$-dimensional feature domain $\mathcal{F}$ denoted by $\mathcal{A}:\mathcal{X}\mapsto \mathcal{F}$ is jointly learned along with $\mathbf{Z}$, yielding the following framework.
\begin{equation}\label{Equ:LatentRegularizedSC}
  \min_{\mathbf{Z},\Theta_{\mathcal{A}}} \sum_{i=1}^N\ell_{\mathcal{F}}\left(\mathcal{A}(\mathbf{x}_i),\mathcal{A}(\mathbf{X}_{\setminus i})\mathbf{z}_i\right)+\alpha J(\mathbf{Z})+\lambda\Omega(\Theta_{\mathcal{A}})
\end{equation}
where $\Omega(\Theta_{\mathcal{A}})$ is a regularization function on transformation parameters $\Theta_{\mathcal{A}}$ and  $\lambda$ is a penalty hyper-parameter. Through this way, the subspace structure encoded in $\mathbf{Z}$ is able to be preserved in the feature domain, especially \eqref{Equ:LatentRegularizedSC} can also achieve the goal of DR if $p<n$.
\subsection{Discriminative Transformation Based Clustering}
\par Note from the above presentations that only the structure information is considered in a generative framework \eqref{Equ:LatentRegularizedSC} and it ignores the underlying influence of subspace membership. It follows that the resulted solutions will be intuitively less discriminative than the supervised or semi-supervised one due to lack of label information \cite{Wen2016,Fang2016,Wen2017,Wen2017a}. Apart from SC, another type of related algorithms termed as discriminative clustering have been developed for general data clustering task \cite{VasileiosZografos2013,Xu2004Maximum}, where the latent variable $\mathbf{q}$ containing the label membership is optimized as:
  \begin{equation}\label{Equ:DiscriminativeSC}
   \max_{\Theta_{\mathcal{A}},\mathbf{q}} p(\mathbf{X}|\mathbf{q},\Theta_{\mathcal{A}})
 \end{equation}
In this spirit, Ding and Li combined linear discriminant analysis (LDA) \cite{Balakrishnama1998} and K-means together to develop a discriminative K-means algorithm based on the criterion of Fisher \cite{Ding2007Adaptive,Ye2007}. Alternatively, in order to maximize the variance of transformed data, another prevalent tool for DR, namely PCA was combined with K-means to develop a discriminative embedded clustering framework \cite{Hou2015}. Nevertheless, the above two discriminative criterions are developed for general clustering task without further exploiting the properties of subspace so that they are normally not appropriate for SC. Accordingly, some suitable criterions tailored for subspace should be newly considered \cite{Wen2016,zhang2013low,Feng2014Robust}. To this end, Qiu and Sapiro specifically developed a novel subspace oriented regularization as following.
\begin{equation}\label{Equ:DiscrminiativeRegularization}
  \min_\mathbf{A} \Omega(\mathbf{X},\mathbf{q}|\mathbf{A})=\sum_{k=1}^K\|\mathbf{AX}^k\|_*-\|\mathbf{AX}\|_*,~\mathrm{s.t.}~\|\mathbf{A}\|_2=1
\end{equation}
where $\Theta_{\mathcal{A}}=\{\mathbf{A}\}$ is the linear matrix, $\|\cdot\|_*$ is the nuclear norm and $\mathbf{X}^k$ contains the samples in the $k$-th clustering group according to $\mathbf{q}$. They proved that this regularizer will achieve its lower bound when each pair of subspace will be orthogonal so that the transformed data $\mathbf{AX}$ will be encouraged to appear a maximum separation structure among the inter-class subspaces. Then they conduct a robust SSC method for clustering in E step, which can produce much better performances that other SSC based variants when $\mathbf{A}$ is a square transformation.
\par Due to lack of subspace structure information in $\mathbf{Z}$, the above difference of convex regularization function in \eqref{Equ:DiscrminiativeRegularization} will, however, admit many trivial solutions especially in the case of a fat matrix $\mathbf{A}$, namely $p<n$. To illustrate this issue more concretely, we will show several examples. Firstly, when each sample in $\mathbf{X}$ will reside in the null space of $\mathbf{A}$, we will have $\Omega(\mathbf{A})=0$ as $\mathbf{Ax}_i=\mathbf{0}$, in which case SC cannot be performed in that feature space. This situation can be also extended as $\mathbf{A}(\mathbf{x}_i-\mathbf{x}_j)=\mathbf{0}$ so that many transformed data will readily collapse onto a specific point, yielding an over-fitting solution \cite{Bishop2006}. Secondly, when $\mathbf{A}$ is rank defective, $\mathbf{AX}$ will likely lose information, which will potentially destroy the intrinsic subspace structures so as to degrade the SC performance. Additionally, these discriminative clustering frameworks will mostly suffer from the mis-clustered outliers in $\mathbf{q}$ as well as noise during iteration. Due to lack of structure and true label information, once many false memberships are coming across in some iteration, the following operator learning will be generally biased and give rise to a domino risk of performance collapsing. To relieve these shortages, the authors exploited an extra robust PCA (RPCA) algorithm \cite{Candes2011} during each iteration to get rid of some outliers but it will be of low efficiency and will increase the model complexity at the same time.

\section{Discriminative Transformation Learning for Fuzzy Sparse Subspace Clustering}\label{Sec:RSSSC}
\par In this section, we will present a novel framework in the hope of effectively addressing all issues mentioned above. To this end, what we actually concern should concentrate on the following model containing two latent variables:
 \begin{equation}\label{Equ:UnitiveSC}
 \begin{split}
   \max_{\mathbf{q},\mathbf{Z},\Theta_{\mathcal{A}}}p\left(\mathbf{q},\mathbf{Z},\Theta_{\mathcal{A}}|\mathbf{X}\right) \propto p\left(\mathbf{X}|\Theta_{\mathcal{A}},\mathbf{q},\mathbf{Z}\right)p\left(\mathbf{q},\mathbf{Z}\right)p(\Theta_{\mathcal{A}})
   \end{split}
 \end{equation}
To explicitly characterize \eqref{Equ:UnitiveSC}, two modules of fuzzy sparse subspace clustering (FSSC) and discriminative transformation learning (DTL) will be developed in our proposed framework. Roughly speaking, in the first module of FSSC, the latent variables of fuzzy labels as well as the self-representations are jointly optimized. Then the module of DTL will focus on learning a discriminative and structure preserved linear transformation operator, in which another local latent variables will be involved to enhance the robustness. We conclude this section by proposing a theoretical comparison analysis with the other related SC algorithms. Before we start, the flow diagram will be firstly illustrated in Fig. \ref{Fig_framework} as an overview.
\begin{figure}
  \centering
  \includegraphics[width=0.5\textwidth]{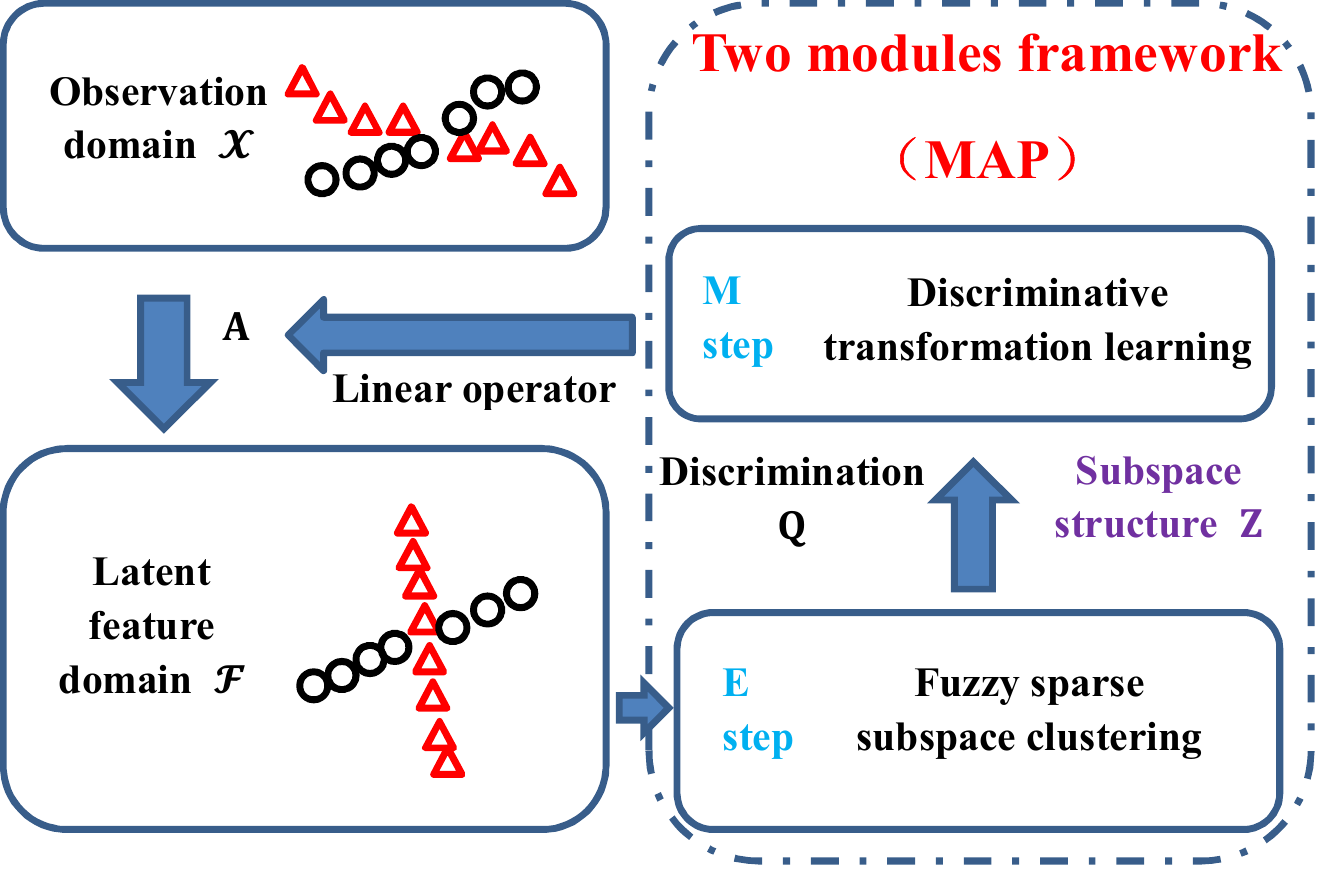}\\
  \caption{Flow diagram of the proposed two modules framework of discriminative transformation learning for fuzzy sparse subspace clustering. }\label{Fig_framework}
\end{figure}

\subsection{Fuzzy Sparse Subspace Clustering}
\par Let $\mathbf{Q}=[\mathbf{q}_1^\mathrm{T};\dots;\mathbf{q}_N^\mathrm{T}]\in\mathbb{R}_+^{N\times K}$ be the fuzzy membership matrix served as a set of latent variables, where $\mathbf{q}_i\in\mathbb{R}_+^{K}$ is a fuzzy non-negative label vector for $i$-th observation arranged as the $i$-th row vector in $\mathbf{Q}$. Considering the optimization in \eqref{Equ:UnitiveSC}, given an operator $\mathcal{A}$ with parameter $\Theta_{\mathcal{A}}$ and a set of data $\{\mathbf{x}_i\}_{i=1}^N$ to be clustered, FSSC will address the following regularized minimization problem:
\begin{equation}\label{EquMAPFSSC}
\begin{split}
  \min_{\mathbf{Q,z}_i} \sum_{i=1}^N \ell_{\mathcal{F}}\left(\mathcal{A}(\mathbf{x}_i),\mathcal{A}(\mathbf{X}_{\setminus i})\mathbf{z}_i\right) +J(\mathbf{z}_i|\mathbf{Q})\\
  \mathrm{s.t.}~\mathbf{Q}\in\{\mathbf{Q}|\mathbf{q}_i\in\Delta^{K-1},~\mathrm{rank}(\mathbf{Q})=K\}.
  \end{split}
\end{equation}
where $J(\mathbf{z}_i|\mathbf{Q})$ is a regularizer on representations depending on $\mathbf{Q}$, each $\mathbf{q}_i$ is restricted in a $K-1$-simplex $\Delta^{K-1}$ for fuzzy consideration, and a full rank constraint is necessarily required to ensure the $K$-group clustering solution. With this constraint, the probability of $\mathbf{x}_i$ belonging to the $k$-th subspace will be encoded in  $\mathbf{q}_i(k)$. To make the full rank constraint easier to handle, we make a relaxation by involving a log barrier regularization on $\mathbf{Q}$ with parameter $\tau$ to prevent its singular values from being zeros. Then \eqref{EquMAPFSSC} will be instead with:
\begin{equation}\label{EquMAPFSSC_log}
\begin{split}
  \min_{\mathbf{Q,z}_i} \sum_{i=1}^N \ell_{\mathcal{F}}\left(\mathcal{A}(\mathbf{x}_i),\mathcal{A}(\mathbf{X}_{\setminus i})\mathbf{z}_i\right) +J(\mathbf{z}_i|\mathbf{Q})-\tau\log\det(\mathbf{Q}^\mathrm{T}\mathbf{Q})\\
  \mathrm{s.t.}~\mathbf{q}_i\in\Delta^{K-1}.
  \end{split}
\end{equation}
Considering $J(\mathbf{z}_i|\mathbf{Q})$, it will be designed in the following way. According to the previous discussions, $|\mathbf{z}_i(j)|$ will reflect the similarity between $\mathbf{x}_i$ and $\mathbf{x}_j$ in terms of subspace structure. If $\mathbf{x}_i$ and $\mathbf{x}_j$ belongs to the same subspace with a high affinity, $|\mathbf{z}_i(j)|$ and $|\mathbf{z}_{j}(i)|$ will be expected to be the large values. To meet this purpose, we can develop a manifold preserving function to model $J(\cdot)$. More formally, if $\mathbf{q}_i$ is near to $\mathbf{q}_j$ in the label domain, the value of $|\mathbf{z}_{i}(j)|$ and $|\mathbf{z}_{j}(i)|$ should be also large, and vice-versa. Therefore, $J(\mathbf{z}_{i}|\mathbf{Q})$ can be designed as the following prevalent manifold preserving regularization \cite{Zheng2013,Li2016}.
\begin{equation}\label{Equ:ManifoldPreserving}
  J(\mathbf{z}_{i}|\mathbf{Q})=\sum_{j\neq i}\left(\beta\|\mathbf{q}_i-\mathbf{q}_j\|^2_2+\alpha\right)|\mathbf{z}_i(j)|
\end{equation}
where $\alpha$ and $\beta$ are two hyper-parameters respectively controlling the bias and smoothness of preservation and we let $\mathbf{z}_{i}(i)=0,\forall i$ all the way through. We can see from \eqref{Equ:ManifoldPreserving} that it is essentially a weighted $\ell_1$ norm where the weight for each entry $\mathbf{z}_i(j)$ will be controlled by the fuzzy label vectors. Combining \eqref{Equ:ManifoldPreserving} and \eqref{EquMAPFSSC_log}, the optimization for FSSC can be finally summarized as
\begin{equation}\label{EquMAPFSSC_final}
\begin{split}
  \min_{\mathbf{Q,z}_i} \sum_{i=1}^N \ell_{\mathcal{F}}\left(\mathcal{A}(\mathbf{x}_i),\mathcal{A}(\mathbf{X}_{\setminus i})\mathbf{z}_i\right)-\tau\log\det(\mathbf{Q}^\mathrm{T}\mathbf{Q})\\
    +\sum_{i=1}^N\sum_{j\neq i}\left(\beta\|\mathbf{q}_i-\mathbf{q}_j\|^2_2+\alpha\right)|\mathbf{z}_i(j)|,\\
    \mathrm{s.t.}~\mathbf{q}_i\in\Delta^{K-1},~\mathbf{z}_i(i)=0.
  \end{split}
\end{equation}
It can be observed from \eqref{EquMAPFSSC_final}, our fuzzy scheme is specifically tailored for the task of SC which is different from the general fuzzy clustering methods \cite{Chang2017}.
\subsection{Discriminative Transformation Learning}
With the learned the fuzzy label matrix $\mathbf{Q}$ as well as the self-representation vectors $\{\mathbf{z}_i\}_{i=1}^N$, we will now address the issue of learning a discriminative and structure preserving transformation operator. To make the model simplicity, we restrict our focus on learning a linear operator $\Theta_{\mathcal{A}}=\{\mathbf{A}\in\mathbb{R}^{p\times n},~p\leq n\}$. Towards this end, we will address the rest problem in \eqref{Equ:UnitiveSC} with respect to $\mathbf{A}$,
 \begin{equation}\label{Equ:MLE_A}
 \begin{split}
   \max_{\mathbf{A}}p\left(\mathbf{X}|\mathbf{Q},\mathbf{Z},\mathbf{A}\right)p(\mathbf{A})&=p(\mathbf{A})\sum_{\mathbf{F}} p(\mathbf{X|F,Z,A})p(\mathbf{F}|\mathbf{Q})\\
   \end{split}
\end{equation}
where we moreover involve a set of latent feature vectors in $\mathbf{F}$ for each sample for the sake of improving the stability and robustness \cite{Bishop2006}. However, dealing with the right hand side of \eqref{Equ:MLE_A} will be difficult. Instead, we will make a relaxation once more by maximizing its lower bound given by:
 \begin{equation}\label{Equ:MLE_AF}
 \begin{split}
   \max_{\mathbf{A},\mathbf{F}} p(\mathbf{A}) p(\mathbf{X|F,Z,A})p(\mathbf{F}|\mathbf{Q})\\
   \end{split}
\end{equation}
which can be further reformulated as the following regularized optimization.
\begin{equation}\label{Eq:DTL_final}
\begin{split}
    \min_{\mathbf{A},\mathbf{F}} \sum_{i=1}^N\left(\frac{\lambda}{2}\|\mathbf{Ax}_i-\mathbf{f}_i\|_2^2+ \ell_{\mathcal{F}}\left(\mathcal{A}(\mathbf{x}_i),\mathcal{A}(\mathbf{X}_{\setminus i})\mathbf{z}_i\right)\right)
    \\+ \left(\sum_{k=1}^K\|\mathbf{F} \mathrm{diag}(\mathbf{Q}_k)\|_*-\|\mathbf{F}\|_*\right)-\tau_1\log\det(\mathbf{AA}^\mathrm{T})
    \end{split}
\end{equation}
where a log barrier prior inducing regularization term is imposed to model $p(\mathbf{A})$ for the sake of preserving the information and overcoming the rank defective shortage in conventional LRT \eqref{Equ:DiscrminiativeRegularization}, $\tau_1$ and $\lambda$ are two regularization hyper-parameters. Comparing \eqref{Eq:DTL_final} with \eqref{Equ:DiscrminiativeRegularization}, our above formulation has a clear and more convinced interpretation of enhancing the robustness from the following three perspectives. Firstly, we impose a subspace oriented discriminative regularization on the latent features $\mathbf{F}$ aiming to find a set of regularized points around the projections $\{\mathbf{Ax}_i\}_{i=1}^N$ while regularization in \eqref{Equ:DiscrminiativeRegularization} will be directly imposed on these projections. Consequently, influence of mis-clustered outliers or noise on $\mathbf{A}$ will be relieved in \eqref{Eq:DTL_final} and more benefits of such type of regularization have been revealed in \cite{Wen2017a}. Furthermore, because of taking into account of $\mathbf{z}_i$, $\mathbf{A}$ will be able to preserve the subspace structure in the feature domain. Finally, we involve a soft membership of fuzzy label matrix in the regularizer instead of the previous hard grouping, which will actually weight each projection with the probability. With this setting, its robustness to outliers will be empirically improved than \eqref{Equ:DiscrminiativeRegularization} without need extra algorithms such as RPCA and its discriminability will be guaranteed by the following theorem.
\begin{theorem}
Let $\mathbf{Q}\in\mathbb{R}_+^{N\times K}$ be a fuzzy label matrix defined as aforementioned and $\mathbf{Q}_k$ be its $k$-th column. Then we have
\begin{equation}\label{Eq:LowboundDiscriminative}
  \sum_{k=1}^K\left\|\mathbf{F} \mathrm{diag}(\mathbf{Q}_k)\right\|_*-\left\|\mathbf{F}\right\|_*\geq 0
\end{equation}
when $\mathbf{Q}$ is a strict binary full rank matrix and the column spaces of $\mathbf{F}\mathrm{diag}(\mathbf{Q}_k),~k=1,\dots,K$ are orthogonal to each other, \eqref{Eq:LowboundDiscriminative} can reach its lower bound.
\end{theorem}
\begin{proof}
As $\mathbf{Q}$ is the fuzzy label matrix, we have $\sum_{k}\mathrm{diag}(\mathbf{Q}_k)=\mathbf{I}$ and $\mathbf{I}$ denotes by the identity matrix. It follows that
\begin{equation}\label{Equ:proof}
  \|\mathbf{F}\|_*=\left\|\mathbf{F}\sum_{k=1}^K \mathrm{diag}(\mathbf{Q}_k)\right\|_*\leq \sum_{k=1}^K\left\|\mathbf{F}\mathrm{diag}(\mathbf{Q}_k)\right\|_*
\end{equation}
Then according to Theorem 2 in \cite{Qiu2014}, we can conclude that its lower bound can be achieved when the column spaces of $\mathbf{F}\mathrm{diag}(\mathbf{Q}_k)$ are orthogonal to each other, which could only happen when $\mathbf{Q}$ is a binary label matrix.
\end{proof}

\subsection{Framework Analysis}
\par In the previous subsections, we have implemented a complete two-modules framework of DTL-FSSC. This subsection will present a deep analysis to highlight its superiorities and the connections as well as distinctions with the other state-of-the-art SC methods.
\par The proposed two modules framework of DTL-FSSC can be regarded as an Expectation-Maximization like scheme for computing the model parameter $\Theta_{\mathcal{A}}$ containing two latent variables, namely $\mathbf{Z}$ and $\mathbf{Q}$. The module of FSSC, corresponding to the E-step, essentially attempts to evaluate the posterior distribution $p(\mathbf{Z},\mathbf{Q}|\mathbf{X},\Theta_{\mathcal{A}})$ with the fixed $\Theta_{\mathcal{A}}$. In the previous related frameworks, e.g., Low rank representation (LRR) \cite{Liu2013}, SSC \cite{Elhamifar2013} and their numerous subsequent variants, only a marginal distribution $p(\mathbf{Z}|\mathbf{X},\Theta_{\mathcal{A}})=\sum_{\mathbf{Q}}p(\mathbf{Z},\mathbf{Q}|\mathbf{X},\Theta_{\mathcal{A}})$ is computed. More specifically, FSSC will involve a conditional distribution $p(\mathbf{Z|Q})$ while others will concentrate on $p(\mathbf{Z})$. Note from this observation apparently that FSSC will be a specific instance of SSC to enhance the discrimination of $\mathbf{Z}$ by conditioning it on the latent label $\mathbf{Q}$ and it can reasonably yield a better clustering performance. After estimating the posterior of $\mathbf{Z}$ and $\mathbf{Q}$, M-step of DTL will update the parameter $\Theta_{\mathcal{A}}$ to learn a transformation $\mathcal{A}:\mathcal{X}\rightarrow \mathcal{F}$ so that clustering in the transformed space can be more readily to realize. In LRT \cite{Qiu2014} and LS3C \cite{Patel2015Latent}, they attempt to deal with the likelihoods $p(\mathbf{X}|\mathbf{Q},\Theta_{\mathcal{A}})$ and $p(\mathbf{X}|\mathbf{Z},\Theta_{\mathcal{A}})$, respectively, and thus the subspace structure or discrimination will be omitted in these two frameworks. As a result, the resulted operator will be heavily influenced by mis-clustered $\mathbf{Q}$ or less discriminative. On the contrary, in DTL-FSSC, we impose a prior of $\mathbf{Q}$ that each label vector should be constrained in a simplex, which will alleviate the influence of some outliers. More importantly, because of the involved $\mathbf{Z}$ in DTL-FSSC, the resulted operator is able to preserve the subspace structure so that the above mentioned trivial solutions can be avoided.

\section{Optimization and Discussion}\label{Sec:OPT}
\par In this section, a detailed optimization scheme will be firstly presented to solve the proposed DTL-FSSC framework, in which $\mathbf{A}$, $\{\mathbf{z}_i\}_{i=1}^N$ and $\mathbf{Q}$ will be regarded as the communication variables for information exchange between two modules. Then the computational complexities of the proposed methods for each step will be analyzed. We conclude this section by some remarks to declare some detailed implementations.
\subsection{Fuzzy Sparse Subspace Clustering}
\par Given a transformation operator $\mathbf{A}$, the problem of FSSC in the corresponding feature domain will be given by
\begin{equation}\label{Equ:FSSC}
\begin{split}
  \min_{\mathbf{Q,z_i}} \sum_{i=1}^N \frac{1}{2}\|\mathbf{\widetilde{x}}_i-\mathbf{\widetilde{X}}_{\setminus i}\mathbf{z}_i\|_2^2 +\sum_{j\neq i}\left(\beta\|\mathbf{q}_i-\mathbf{q}_j\|^2_2+\alpha\right)\left|\mathbf{z}_i(j)\right|
  \\ -\tau\log\det(\mathbf{Q}^\mathrm{T}\mathbf{Q})~\mathrm{s.t.}~\mathbf{q}_i\in\Delta^{K-1}.
  \end{split}
\end{equation}
where $\mathbf{\widetilde{x}}_i=\mathbf{Ax}_i$ is the projection vector and the loss function in feature domain is particularly considered as the square $\ell_2$ norm for the sake of convenience. This optimization is a typical bi-convex problem and we alternatively deal with one variable with the other one fixed.
\subsubsection{\textbf{Update Representations} $\mathbf{z}_i$}
Considering the problem with respect to $\mathbf{z}_i$, it can be rewritten as the following independent weighted $\ell_1$ minimization
\begin{equation}\label{Equ:FSSC_Z}
   \min_{\mathbf{z}_i}\frac{1}{2}\|\mathbf{\widetilde{x}}_i-\mathbf{\widetilde{X}}_{\setminus i}\mathbf{z}_i\|_2^2 +\sum_{j\neq i}\left(\beta\|\mathbf{q}_i-\mathbf{q}_j\|^2_2+\alpha\right)|\mathbf{z}_i(j)|
\end{equation}
which can be solved by many existing solvers such as fast iterative shrinkage threshold algorithm (FISTA) \cite{Beck2009} and thus we will not give a detailed derivation for simplicity.
\subsubsection{\textbf{Update Fuzzy Label Matrix} $\mathbf{Q}$}
The problem with respect to $\mathbf{Q}$ can be reformulated as the following constraint smooth convex optimization which can be typically solved by projected gradient method or Frank-Wolfe of conditional gradient method \cite{Jaggi2013Revisiting}.
\begin{equation}\label{Equ:FSSC_Q}
  \min_{\mathbf{Q}} \mathrm{Tr}(\mathbf{Q}^\mathrm{T} \mathbf{L}_Z\mathbf{Q})-\widetilde{\tau}\log\det(\mathbf{Q}^\mathrm{T}\mathbf{Q})~\mathrm{s.t.}~\mathbf{q}_i\in\Delta^{K-1}
\end{equation}
where $\widetilde{\tau}=\frac{\tau}{\beta}$, $\mathbf{L}_Z$ is the graph Laplacian constructed from $\mathbf{Z}$. More concretely, $\mathbf{L}_Z=\mathbf{D}_Z-\mathbf{W}_Z$ and $\mathbf{D}_Z$ is a diagonal matrix with entries $\mathbf{D}_Z(j,j)=\sum_{i}\mathbf{W}_{i,j}$ and $\mathbf{W}_{Z}=\frac{1}{2}(|\mathbf{Z}|+|\mathbf{Z}^\mathrm{T}|)$. In this paper, we will adopt the alternating direction method of multiplier (ADMM) to solve this problem efficiently by introducing an auxiliary variable $\widehat{\mathbf{Q}}$ for $\mathbf{Q}$ \cite{Boyd2011}. Then we have the following iterative scheme
\begin{equation}\label{SolveQ}\small
  \begin{cases}
  \min_\mathbf{Q} \mathrm{Tr}(\mathbf{Q}^\mathrm{T} \mathbf{L}_Z\mathbf{Q})-\widetilde{\tau}\log\det(\mathbf{Q}^\mathrm{T}\mathbf{Q})+\frac{\rho}{2}\|\mathbf{Q}-\mathbf{\widehat{Q}}^{t}+\mathbf{\Lambda}^{t}/\rho\|_\mathrm{F}^2\\
  \min_{\mathbf{\widehat{Q}}} \|\mathbf{Q}^{t+1}-\mathbf{\widehat{Q}}+\mathbf{\Lambda}/\rho\|_\mathrm{F}^2~\mathrm{s.t.}~\mathbf{\widehat{q}}_i\in\Delta^{K-1}\\
  \mathbf{\Lambda}^{t+1}\gets \mathbf{\Lambda}^{t}+\rho(\mathbf{Q}^{t+1}-\mathbf{\widehat{Q}}^{t+1})
\end{cases}
\end{equation}
where the superscript $t$ or $t+1$ denotes the iteration counts in the optimization of $\mathbf{Q}$, $\mathbf{\Lambda}$ is the Lagrangian multiplier and $\rho>0$ is the penalty parameter.  We will next show that both of the subproblems in \eqref{SolveQ} admit the closed form solutions.
\par Firstly, we reformulate the subproblem with respect to $\mathbf{Q}$ as
\begin{equation}\label{Opt:subQ}
  \min_\mathbf{Q} \mathrm{Tr}\left(\mathbf{Q}^\mathrm{T} (\mathbf{L}_Z+\frac{\rho}{2}\mathbf{I})\mathbf{Q}\right)-\widetilde{\tau}\log\det(\mathbf{Q}^\mathrm{T}\mathbf{Q})-\rho \mathrm{Tr}(\mathbf{Q}^\mathrm{T}\widetilde{\mathbf{\Lambda}})
\end{equation}
where $\widetilde{\mathbf{\Lambda}}=\mathbf{\widehat{Q}}-\mathbf{\Lambda}/\rho$. Since any a graph Laplacian is a semi-definite matrix, $\mathbf{L}_Z+\frac{\rho}{2}\mathbf{I}$ will be a positive definite matrix admitting Cholesky factorization as $\mathbf{L}_Z+\frac{\rho}{2}\mathbf{I}=\mathbf{G}^\mathrm{T}\mathbf{G}$. Let $\mathbf{Q}=\mathbf{G}^{-1}\mathbf{B}$, \eqref{Opt:subQ} will be equivalent to
\begin{equation}\label{Opt:subQ1}
  \min_\mathbf{B} \mathrm{Tr}(\mathbf{B}^\mathrm{T}\mathbf{B})-\tau\log\det(\mathbf{B}^\mathrm{T}\mathbf{B})-\rho \mathrm{Tr}(\mathbf{B}^\mathrm{T}\mathbf{G}^\mathrm{-T}\widetilde{\mathbf{\Lambda}})+\mathrm{const}
\end{equation}
If we denote the full singular value decomposition (SVD) of $\mathbf{B}$ and $\mathbf{G}^\mathrm{-T}\widetilde{\mathbf{\Lambda}}$ by $\mathbf{B}=\mathbf{U}\mathbf{\Sigma} \mathbf{V}^\mathrm{T}$ and $\mathbf{G}^\mathrm{-T}\widetilde{\mathbf{\Lambda}}=\mathbf{U}_1\mathbf{\Sigma}_1\mathbf{V}_1^\mathrm{T}$, respectively, \eqref{Opt:subQ1} can be further solved with the following.
\begin{equation}\label{Opt:subQ2}
  \min_\mathbf{\Sigma} \mathrm{Tr}(\mathbf{\Sigma}^\mathrm{T}\mathbf{\Sigma})-\widetilde{\tau}\log\det(\mathbf{\Sigma}^\mathrm{T}\mathbf{\Sigma})-\rho \mathrm{Tr}(\mathbf{V}\mathbf{\Sigma}^\mathrm{T} \mathbf{U}^\mathrm{T}\mathbf{U}_1\mathbf{\Sigma}_1\mathbf{V}_1^\mathrm{T})
\end{equation}
As $\mathrm{Tr}(\mathbf{V}\mathbf{\Sigma}^\mathrm{T} \mathbf{U}^\mathrm{T}\mathbf{U}_1\mathbf{\Sigma}_1\mathbf{V}_1^\mathrm{T})\leq \mathrm{Tr}(\mathbf{\Sigma}^\mathrm{T}\mathbf{\Sigma}_1)$ with the upper bound attaining in the case of $\mathbf{U}=\mathbf{U}_1$ and $\mathbf{V}=\mathbf{V}_1$ \cite{Mirsky1959}, a lower bound of \eqref{Opt:subQ2} can be obtained by solving
\begin{equation}\label{Opt:subQ3}
\begin{split}
  &\min_\mathbf{\mathbf{\Sigma}} \mathrm{Tr}(\mathbf{\Sigma}^\mathrm{T}\mathbf{\Sigma})-\widetilde{\tau}\log\det(\mathbf{\Sigma}^\mathrm{T}\mathbf{\Sigma})-\rho \mathrm{Tr}(\mathbf{\Sigma}^\mathrm{T} \mathbf{\Sigma}_1)\\
  =&\min_{\sigma_i} \sum_{i}\left(\sigma_i^2-\rho\sigma_i s_i-2\widetilde{\tau}\log \sigma_i\right)
  \end{split}
\end{equation}
where $\sigma_i$ and $s_i$ are the $i$-th singular value in $\mathbf{\Sigma}$ and $\mathbf{\Sigma}_1$, respectively. Therefore, \eqref{Opt:subQ3} admits a closed form solution for each $\sigma_i$ by setting its derivation as zero and leaving out the negative solution, which is given by
 \begin{equation}\label{Sol:subQ3}
  \sigma_i^*=\frac{1}{4}(\rho s_i+\sqrt{\rho^2s_i^2+16\widetilde{\tau}})
\end{equation}
The optimal solution to \eqref{Opt:subQ} can be summarized as $\mathbf{Q}^{t+1}=\mathbf{G}^{-1}\mathbf{U}_1\mathbf{\Sigma}^* \mathbf{V}_1^\mathrm{T}$, where $\sigma^*_i$ will be the $i$-th singular value in $\mathbf{\Sigma}^*$.
\par Next, the optimal solution of $\widehat{\mathbf{Q}}$ is simply the projection onto the simplex as
\begin{equation}\label{Sol:Qhat}
  \mathbf{\widehat{Q}}^{k+1}=\max(\mathbf{q}_i+\mathbf{\Lambda}_i/\rho-\nu_i \mathbf{1},\mathbf{0}),
\end{equation}
where $\nu_i$ is a scalar computed by bisection search such that $\mathbf{1}^\mathrm{T}\mathbf{q}_i^{k+1}={1}$ \cite{Parikh2014} and $\mathbf{1}$, $\mathbf{0}$ stand for the full 1 and 0 vectors, respectively.

\subsection{Discriminative Transformation Learning}
For the module of DTL, we will rely on the previous computed $\mathbf{Z}$ and $\mathbf{Q}$ and solve the following optimization in an alternative way.
\begin{equation}\label{Equ:DTL}
\begin{split}
    \min_{\mathbf{A,F}} \sum_{i=1}^N\left(\frac{\lambda}{2}\|\mathbf{Ax}_i-\mathbf{f}_i\|_2^2+\frac{1}{2}\|\mathbf{Ax}_i-\mathbf{Ax}_i\mathbf{z}_i\|_2^2\right)
    \\+\left(\sum_{k=1}^K\|\mathbf{F} \mathrm{diag}(\mathbf{Q}_k)\|_*-\|\mathbf{F}\|_*\right)-\tau_1\log\det(\mathbf{AA}^\mathrm{T})
  \end{split}
\end{equation}
\subsubsection{\textbf{Update Latent Features} $\mathbf{F}$}
We will start from the subproblem with respect to $\mathbf{F}$ which is summarized as
\begin{equation}\label{Equ:F}
  \min_{\mathbf{F}} \frac{\lambda}{2}\|\mathbf{AX-F}\|_\mathrm{F}^2+\left(\sum_{k=1}^K\|\mathbf{F} \mathrm{diag}(\mathbf{Q}_k)\|_*-\|\mathbf{F}\|_*\right).
\end{equation}
Since the regularization function in \eqref{Equ:F} is the difference of convex regularizer, the final output will heavily depend on the initialization of $\mathbf{F}$. To produce a better and reasonable initialization, we will reconsider the behaviour of this regularization. Note from the first term that $\|\mathbf{F} \mathrm{diag}(\mathbf{Q}_k)\|_*$ is actually a low rank promoting regularizer encouraging the intra-class similarities in a fuzzy way \cite{Wen2016}. On the contrary, $-\|\mathbf{F}\|_*$ will help to increase the inter-class dissimilarities \cite{Qiu2014}. According to this interpretation, our plan of searching a better $\mathbf{F}$ will be heuristically described as following. Since the first term is a convex regularizer, we will firstly solve the following problem of low rank matrix approximation with many solvers to obtain a global solution $\widetilde{\mathbf{F}}$.
 \begin{equation}\label{opt:Finit}
  \widetilde{\mathbf{F}}=\arg\min_{\mathbf{F}} f(\mathbf{F})\doteq\frac{\lambda}{2}\|\mathbf{AX-F}\|_\mathrm{F}^2+\sum_{k=1}^K\|\mathbf{F} \mathrm{diag}(\mathbf{Q}_k)\|_*
\end{equation}
According to our interpretation, $\widetilde{\mathbf{F}}$ contains the feature vectors with high intra-class similarities with respect to $\mathbf{Q}$. Then, we will solve the primary difference of convex problem with $\widetilde{\mathbf{F}}$ as the starting point.
  \begin{equation}\label{opt:Ffinal}
 \min_{\mathbf{F}} f(\mathbf{F})-\|\mathbf{F}\|_*
\end{equation}
As $\mathbf{\widetilde{F}}$ is the global minimizer of $f(\mathbf{F})$, if the value of \eqref{opt:Ffinal} can be decreased in a monotonous way, we may find an optimization path that gradually increases the value of $\|\mathbf{F}\|_*$, namely increase the dissimilarities of inter-class features. With this strategy, the intra-class similarities may be preserved as the inter-class variance increased. To meet this requirement, we will exploit the following subgradient descent scheme in a batch mode way starting from $\widetilde{\mathbf{F}}$ such that each step will effectively decrease \eqref{opt:Ffinal}.
\begin{equation}\label{GradientF}
  \mathbf{F}^{t+1}\gets \mathbf{F}^t-\mu \left(\partial f(\mathbf{F}^t)-\partial \|\mathbf{F}^t\|_*\right),~\mathbf{F}^0=\mathbf{\widetilde{F}}.
\end{equation}
where $\partial\|\cdot\|_*$ is the subdifferential of the nuclear norm \cite{Qiu2014} and $\mu$ is the step size. We empirically find that such an optimization plan is rather effective and efficient to obtain a better result than stochastic gradient descent.
\subsubsection{\textbf{Update} $\mathbf{A}$}
Finally, we come to the problem of linear operator update whose corresponding convex problem is formulated as following.
\begin{equation}\label{Equ:DTL_A}
\begin{split}
  \min_{\mathbf{A}} \frac{\lambda}{2}\|\mathbf{AX-F}\|_\mathrm{F}^2+\frac{1}{2}\|\mathbf{AXZ-AX}\|_\mathrm{F}^2-\tau_1\log\det\left(\mathbf{AA}^\mathrm{T}\right)
  \end{split}
\end{equation}
To solve this problem, we also utilize the scheme of ADMM and its optimization process will be similar to that of $\mathbf{Q}$ by solving the following three subproblem alternatively.
\begin{equation}\label{SolveA}
  \begin{cases}
  \min_\mathbf{A} \frac{\lambda}{2}\|\mathbf{AX-F}\|_\mathrm{F}^2+\frac{1}{2}\|\mathbf{AXZ-AX}\|_\mathrm{F}^2\\
  +\frac{\rho}{2}\|\mathbf{A}-\mathbf{\widehat{A}}^{t}+\mathbf{\Lambda}^{t}/\rho\|_\mathrm{F}^2,\\
  \min_{\mathbf{\widehat{A}}} \frac{\rho}{2}\|\mathbf{A}^{t+1}-\mathbf{\widehat{A}}+\mathbf{\Lambda}/\rho\|_\mathrm{F}^2-\tau_1\log\det\left(\mathbf{\widehat{A}\widehat{A}}^\mathrm{T}\right),\\
  \mathbf{\Lambda}^{t+1}\gets \mathbf{\Lambda}^{t}+\rho\left(\mathbf{A}^{t+1}-\mathbf{\widehat{A}}^{t+1}\right)
\end{cases}
\end{equation}
where $\mathbf{\widehat{A}}$ is the auxiliary variable and $\rho$ and $\mathbf{\Lambda}$ are penalty and Lagrangian multiplier. Considering the above subproblems, the first one in \eqref{SolveA} will be a standard quadratic function with respect to $\mathbf{A}$. It follows that we can obtain an optimal solution by setting its derivation as zero, yielding
\begin{equation}\label{Solu:A}
  \mathbf{A}^{t+1}=(\lambda \mathbf{FX}^\mathrm{T}+\rho \mathbf{\widehat{A}}-\mathbf{\Lambda})(\lambda \mathbf{XX}^\mathrm{T}+\mathbf{E}\mathbf{E}^\mathrm{T}+\rho \mathbf{I})^{-1}
\end{equation}
where $\mathbf{E=XZ-X}$ is essentially the representation error of the primary samples. The second subproblem with respect to $\widehat{\mathbf{A}}$ will be identical to \eqref{Opt:subQ} and thus share the same optimization scheme.
\par The overall optimization scheme for two modules of DTL-FSSC will be summarized in Algorithm \ref{Algorithm1}.
\begin{algorithm}[htb]
\caption{DTL-FSSC}\label{Algorithm1}
\begin{algorithmic}[1]
%\REQUIRE ~~\\ %Input
\STATE \textbf{Input}: Observations $\mathbf{X}$; Number of subspace $K$; Regularization parameters: $\alpha$, $\beta$, $\tau$, $\tau_1$, $\lambda$ , Maximum overall iterations for FSSC, DTL and overall framework: $T_{fssc}$, $T_{dtl}$ and $T_{max}$, respectively.
\STATE{ Initialization: $t\gets0$, $\mathbf{A}\gets$ Gaussian random matrix with normalized spectral norm, $\mathbf{Z}\gets$ SSC result.}
\STATE {Main Loop:}
 \WHILE{$t\leq T_{max}$}
  \STATE {Data transformation $\mathbf{\widetilde{X}}=\mathbf{AX}$.}
  \WHILE {$t_{fssc}<T_{fssc}$}
 \STATE {Perform FSSC on $\widetilde{\mathbf{X}}$ to update $\mathbf{Q}$ and $\mathbf{Z}$ alternatively.}\\
   \ENDWHILE
 \STATE{Initialize $\mathbf{F}$ by solving \eqref{opt:Finit}.}\\
 \WHILE {$t_{dtl}\leq T_{dtl}$}
  \STATE {Update $\mathbf{F}$ with \eqref{GradientF} and update $\mathbf{A}$ according the scheme of \eqref{SolveA} alternatively.}\\
  \ENDWHILE
 \ENDWHILE
 \STATE \textbf{Output}: label assignment according to fuzzy $\mathbf{Q}$ as $\mathrm{label}(\mathbf{x}_i)\gets \arg\max_{j}(\mathbf{q}_i(j))$
\end{algorithmic}
\end{algorithm}

\subsection{Computational Complexity Analysis}
Considering our proposed optimization scheme in the module of FSSC. As it is independent for each sample, we can solve all of them in a distributed way. For one sample, when we use the FISTA solver, the computational cost will merely be concentrated on $\mathbf{\widetilde{X}}_{\setminus i}^\mathrm{T}\mathbf{\widetilde{X}}_{\setminus i}$ during computing the gradient whose complexity will be approximately $\mathcal{O}(pN^2)$. However, this term will be unchanged so that it is only required to be computed once. During optimizing $\mathbf{Q}$, it is firstly required to compute Cholesky factorization to obtain $\mathbf{G}^\mathrm{T}$ and its inversion. The rest core computations will then come from the full SVD of $\mathbf{G}^{-\mathrm{T}}\widetilde{\mathbf{\Lambda}}$ whose complexity will be normally $\mathcal{O}(N^3)$. In the module of DTL, we have to firstly compute an elaborated  $\mathbf{\widetilde{F}}$ to produce a better initialization, in which singular vector threshold (SVT) operator will be leveraged for low rank requirement \cite{cai2010singular} with $\mathcal{O}\left(\min(p^2N,pN^2)\right)$ complexity. Then computing the subgradient of nuclear norm for updating $\mathbf{F}$ will also cost $\mathcal{O}\left(\min(p^2N,pN^2)\right)$ by exploiting SVD. We can conclude that when $p\ll n$, it will require fewer computations in the above operations. When we update $\mathbf{A}$ with ADMM, we have to firstly compute the inverse of $(\lambda \mathbf{XX}^\mathrm{T}+\mathbf{E}\mathbf{E}^\mathrm{T}+\rho \mathbf{I})$. This operation will normally cost $\mathcal{O}(n^3)$. However, when $N<n$, we can resort to matrix inverse equation to decrease this complexity $\mathcal{O}(N^3)$. In both cases, this operation will keep unchanged during operator learning. For the auxiliary variable, full SVD of $\mathbf{A}+\mathbf{\Lambda}/\rho$ will be computed with the cost of $\mathcal{O}(p^2n)$. In a summary, the complexity of the global framework is acceptable compared with the other SSC based frameworks but we can still conclude that the complexity is cubical with $N$ so that it will not appropriate for the large scaled SC applications. We empirically observe that it will only spend less than half a minute on a PC for most situations and two minutes for the most extreme cases in our experiments. However, for spectral clustering based approaches, handling the large scaled problem is always an intrinsic issue and it is actually beyond the scope of this paper. We suggest that a possible solution may refer to the strategies in \cite{Tremblay2016,Peng2016a} or some stochastic gradient based optimization schemes.
\subsection{Remarks}
\par In this part, we will present some remarks on some issues for reproducing the framework in practical applications. It is better to perform iterative updating in the module of FSSC, especially when we aim to learn a relatively low dimensional subspace for clustering. As a consequence, we always set $T_{fssc}=3$ to achieve a tradeoff between computational cost and performance. On the contrary, we will set $T_{dtl}=1$ to produce an acceptable performance. In the module of DTL, since we have developed a novel strategy for initializing $\mathbf{F}$, the subsequent subgradient descent scheme will be performed within only 10 iterations in our experiments. For global iterations, we observe that as the dimension of the feature domain decreases, more iterations will be required to produce a stable performance. Therefore, in our experiments, $T_{max}=30$ will be considered for a $p\ll n$ dimensional feature subspace while $T_{max}=10$ will be enough for $p=n$ dimensional domain. Another important issue is that all projections to be clustered in the feature subspace can be normalized with unit $\ell_2$ length to improve the efficiency.

\section{Experiments}\label{Sec:Experiments}
In this section, we will evaluate the proposed framework with extensive experiments. In the first place, the convergences of both modules will be empirically validated, respectively. Then the performances of the regularization parameters will be investigated. Some specific experiments will be designed to establish the effectiveness of our proposed framework. It will be finally compared with the other state-of-the-art approaches for subspace clustering on benchmark databases of three typical applications to demonstrate its superiorities, namely Hopkins 155 dataset\footnote{http://www.vision.jhu.edu/data/hopkins155/} for motion segmentation \cite{Tron2007}, Binary Alphadigits dataset (BAD)\footnote{http://www.cs.nyu.edu/~roweis/data.html} for handwritten digits image clustering \cite{Patel2014Kernel}, Extended Yale B database (EYB) for face image clustering \cite{lee2005acquiring}. All experiments are conducted on a desktop PC equipped with i5 CPU$@$3.2GHz and 8GB memory under the environment of MATLAB-2014b.
\subsection{Initialization and Convergence Analysis}
 \begin{figure*}
  \centering
   \subfigure[]{\includegraphics[width=0.3\textwidth]{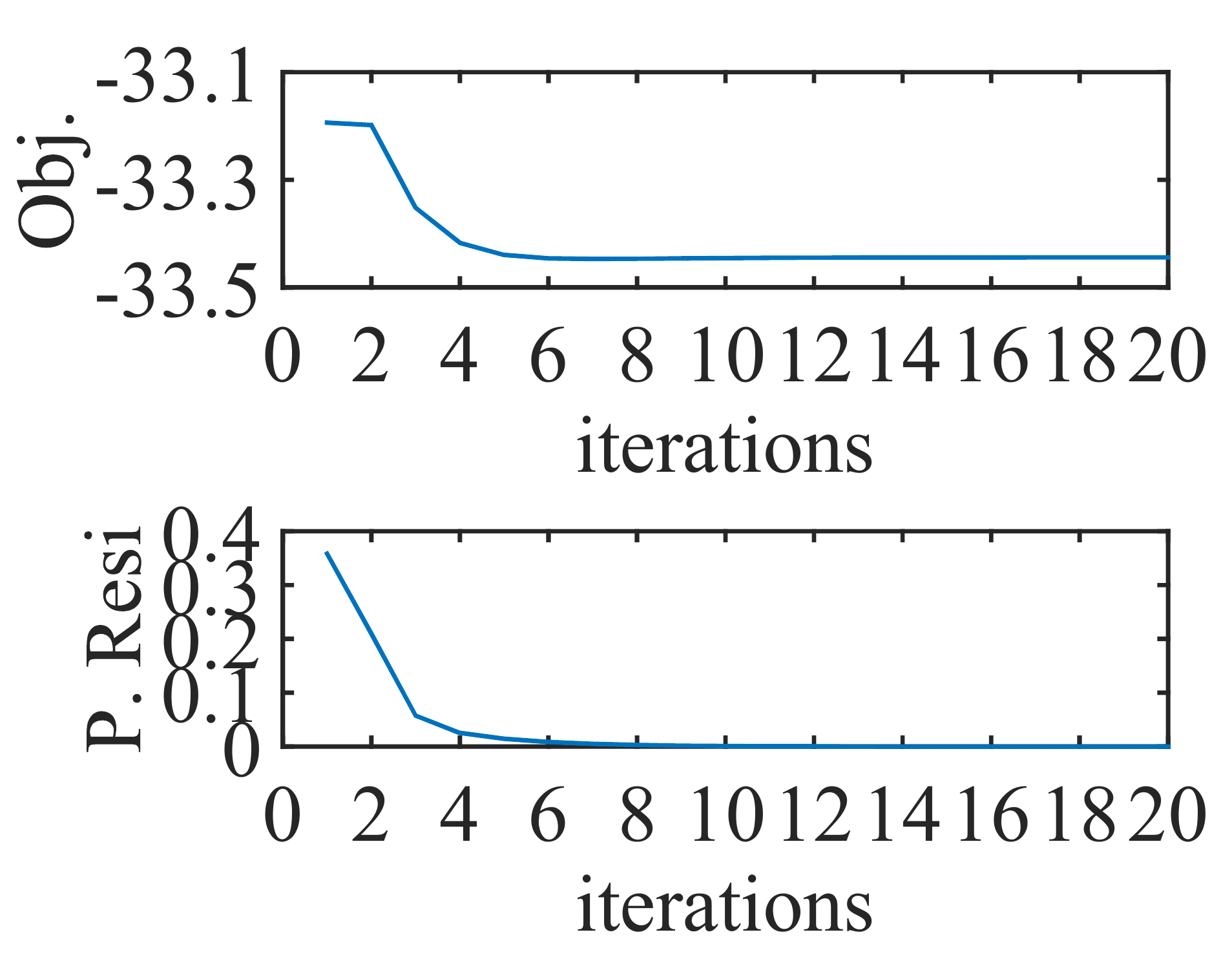}\label{Fig:Conv_Q}}
  \subfigure[]{\includegraphics[width=0.3\textwidth]{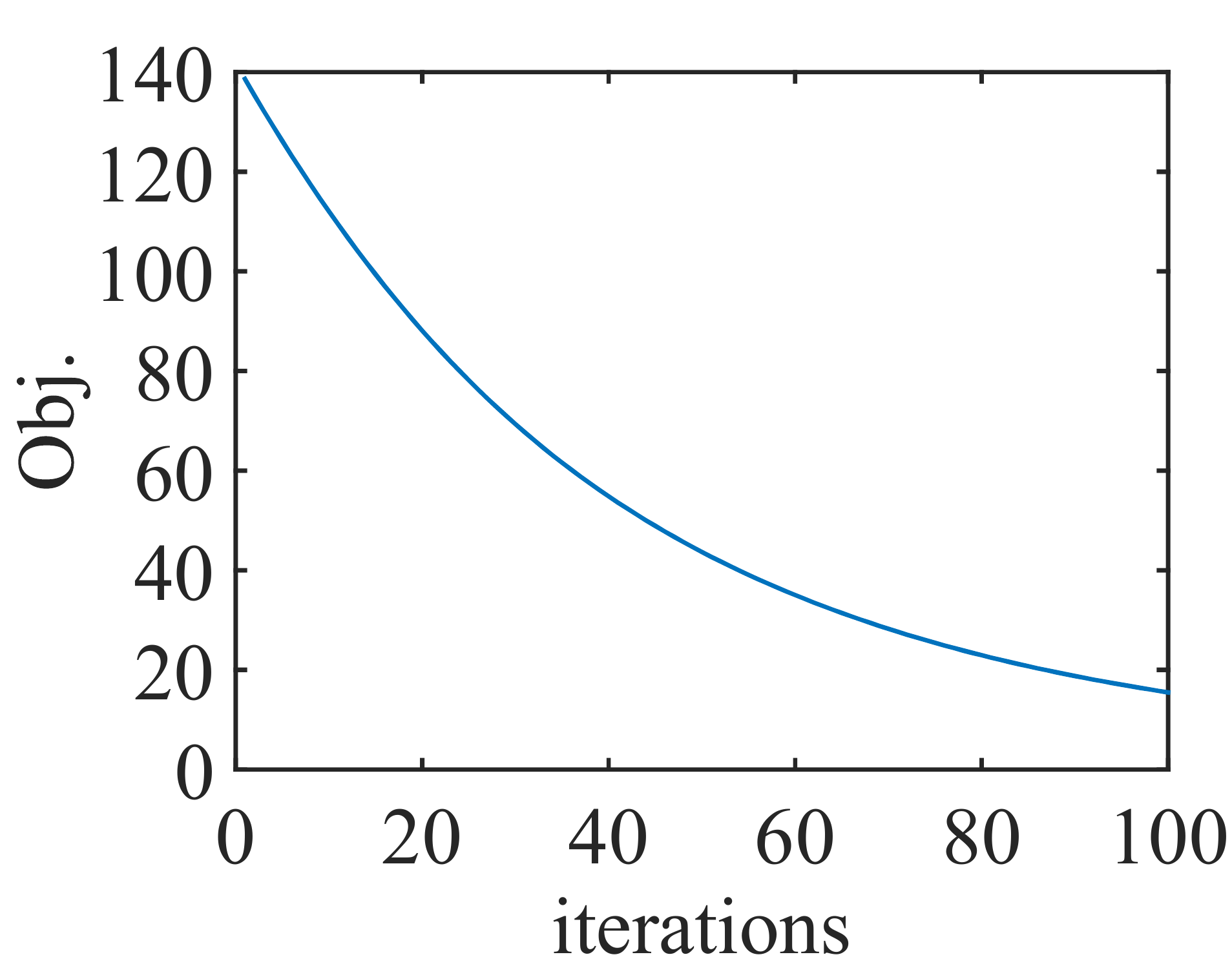}\label{Fig:Conv_F}}
    \subfigure[]{\includegraphics[width=0.3\textwidth]{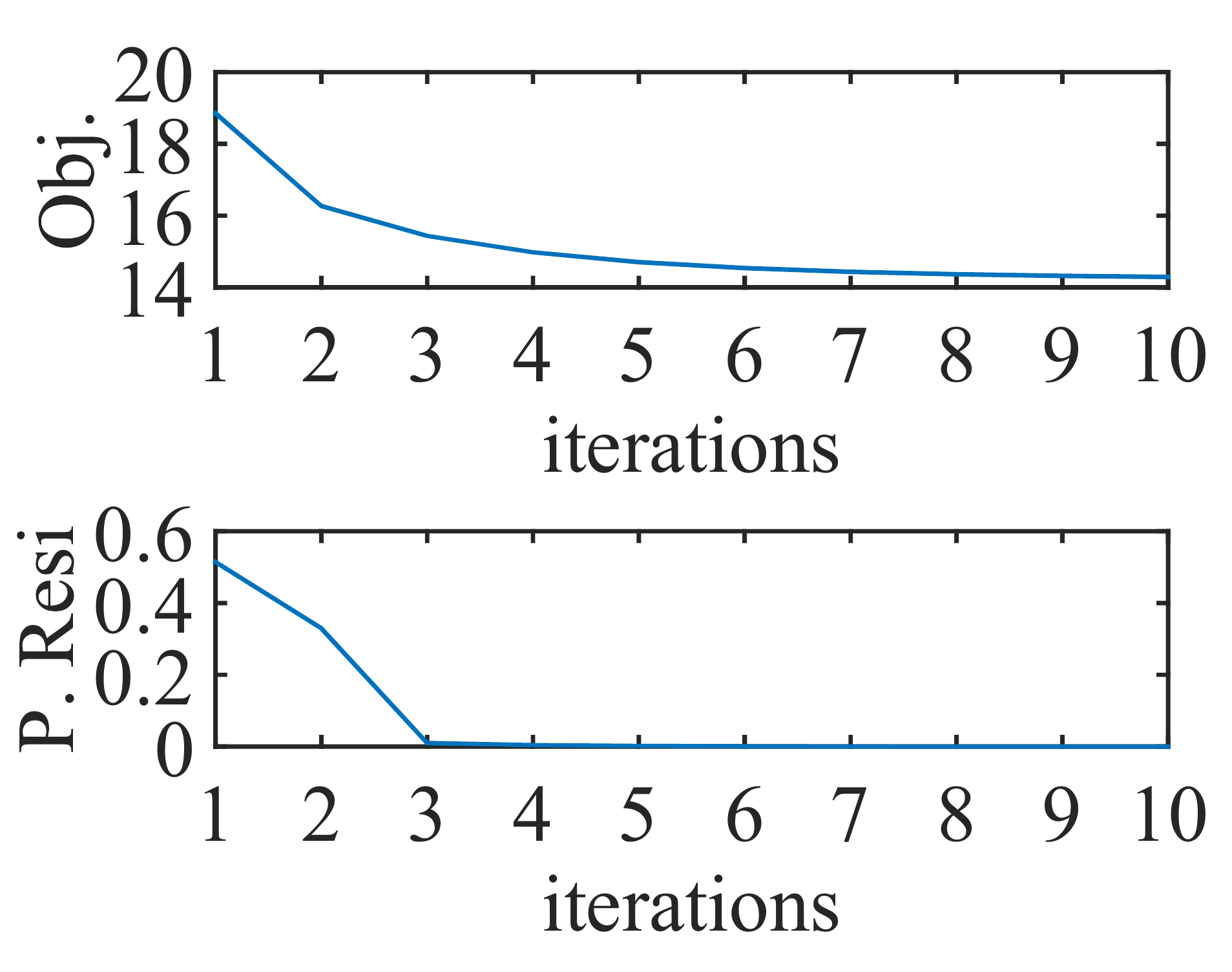}\label{Fig:Conv_A}}
  \caption{Convergence curves v.s. iteration times. (a). Objective function value (Obj.) and primal residual (P.Resi.) of Opt. \eqref{Equ:FSSC_Q}. (b).Obj. of Opt. \eqref{Equ:F}. (c).Obj. and P. Resi of Opt. \eqref{Equ:DTL_A} }\label{Fig:Conve}
\end{figure*}
\par In our framework, $\mathbf{A}$ is initialized as the standard Gaussian matrix and then normalized its spectral norm as 1. $\mathbf{Z}$ will be initialized as the result of the standard SSC on $\mathbf{AX}$. In the phase of FSSC, it is a bi-convex problem, \emph{i.e.,} updating $\mathbf{Z}$ or $\mathbf{Q}$ with other one fixed is a convex problem admitting the global solution. More specifically, if FISTA solver is exploited to update $\mathbf{Z}$, we can receive a fast convergence for objective function \eqref{Equ:FSSC_Z} with rate $\mathcal{O}(1/t^2)$ \cite{Beck2009}. For the subproblem of $\mathbf{Q}$, we adopt ADMM framework to produce a solution whose convergence has been investigated in \cite{Boyd2011}. During the phase of DTL, we exploit a heuristic scheme to choose a starting point of $\mathbf{F}$ by solving \eqref{opt:Finit} which is still a convex problem admitting the global optimal solution. Based on this starting point, we use the subgradient method to converge to a local stationary point $\mathbf{F}$. Finally, the subproblem with respect to $\mathbf{A}$ is similar to that of $\mathbf{Q}$, which is a convex problem with a global optimal result. We empirically check the convergence for each subproblem in which samples from Hopkins 155 dataset will be exploited without loss of the generality. The convergence curves v.s. iterations in their corresponding optimization schemes for updating $\mathbf{Q}$, $\mathbf{F}$ and $\mathbf{A}$ are respectively plotted in Figs. \ref{Fig:Conve}. Inspecting the curves, we may conclude that algorithms for updating $\mathbf{Q}$ and $\mathbf{A}$ will converge within at most 10 iterations while \eqref{Equ:F} will converge slowly due to the subgradient descent scheme. However, it is declared in the previous discussion that it is not needed to reach an exact solution so as to get rid of over-fitting \cite{Bishop2006}. Therefore, during updating $\mathbf{F}$, we empirically conclude that it is better to stop the iteration procedure within 10 times, which have been discussed in the previous part.

\subsection{Parameter Analysis}
\par In the two modules framework of DTL-FSSC, three and two regularization hyper-parameters are respectively involved in FSSC and DTL. In this part, the behaviours of these parameters will be empirically validated on Binary Alphadigits dataset, in which three digits are randomly selected for clustering. Before we start, we need some criterion to evaluate the clustering performance \cite{Strehl2002}. Following the common criterion in the previous SC researches \cite{Elhamifar2013}, we will also leverage the clustering error rate to evaluate the performance which is defined as:
\begin{equation}\label{Equ:Missrate}
  \mathrm{clustering~error} = \frac{\# \mathrm{misclassified~samples}}{\mathrm{total}~\#~\mathrm{of~samples}}\times 100\%
\end{equation}
\par We firstly keep the parameters of $\lambda$ and $\tau_1$ in DTL module unchanged, and investigate $\alpha$, $\beta$ and $\tau$. To do so, we firstly fix $\tau$ and vary $\alpha$ and $\beta$ from $0.01$ to $1$, respectively. The clustering performances are plotted in Fig. \ref{Fig:Para_alphabetamean}. These two regularization parameters both control the sparsity of $\mathbf{Z}$ will determine the similarities (magnitude) and connections (zero or not) among each feature point. To be more concrete, under the assumption of linear subspace, a larger value of $\alpha$ or $\beta$ will encourage less false connections but the true connections will be simultaneously decreased, which will not be beneficial for the clustering \cite{Soltanolkotabi2014}. We can see from the results that when the bias parameter $\alpha$ is set as a small value, \emph{e.g.}, 0.01, the clustering error will be relatively large no matter how $\beta$ varies. When the value of $\alpha$ is increasing, the performance will be significantly promoted. Compared with $\alpha$, the influence of $\beta$ is weak but we will still need a suitable value to achieve a better performance. Next, we fix $\alpha$, $\beta$ and vary $\tau$ from $1.5$ to $10$ and plot the result in Fig. \ref{Fig:Para_tau}. This regularization parameter for log barrier term will control the rank of membership matrix $\mathbf{Q}$. To provide a full rank output, a larger value of $\tau$ is always preferred, otherwise some columns of $\mathbf{Q}$ will be zero, yielding a defective number of clusters. However, too large a value will generate an over-penalty result and the clustering performance will be also degraded. Accordingly, a proper value should be carefully tuned for different datasets. See from the curve that the clustering error will firstly decrease as the increase of $\tau$ and then increase in the case of $\tau>7.5$, which indeed accounts for our previous analysis. Finally, we vary the regularization parameters $\lambda$ and $\tau_1$ in DTL module from $0.01$ to $1$ and $0.001$ to $1$, respectively. Since $\lambda$ controls the weight of discriminative regularization, a proper value is always needed to achieve a better discrimination. Intuitively speaking, a small value of $\lambda$ will produce a more separated feature vectors with respect to $\mathbf{Q}$. Nevertheless, since we have no accurate supervised label information, a fuzzy membership in $\mathbf{Q}$ is not always correct. As a result, too large value of $\lambda$ will sometimes result in a biased discrimination so as to degrade the performance contrarily. Seeing from the results in Fig. \ref{para_lambdatau1}, the rates are oscillation during varying $\lambda$. However, we can still observe that the optimal value for $\lambda$ will be around 0.09 for this database. For $\tau_1$, its impact on clustering error is not apparent. In our following experiments, $\tau_1$ will be always set within 0.001 to 0.07 to reach a satisfactory performance and other hyper-parameters will be tuned for each database, respectively.
 \begin{figure*}
  \centering
   \subfigure[]{\includegraphics[width=0.3\textwidth]{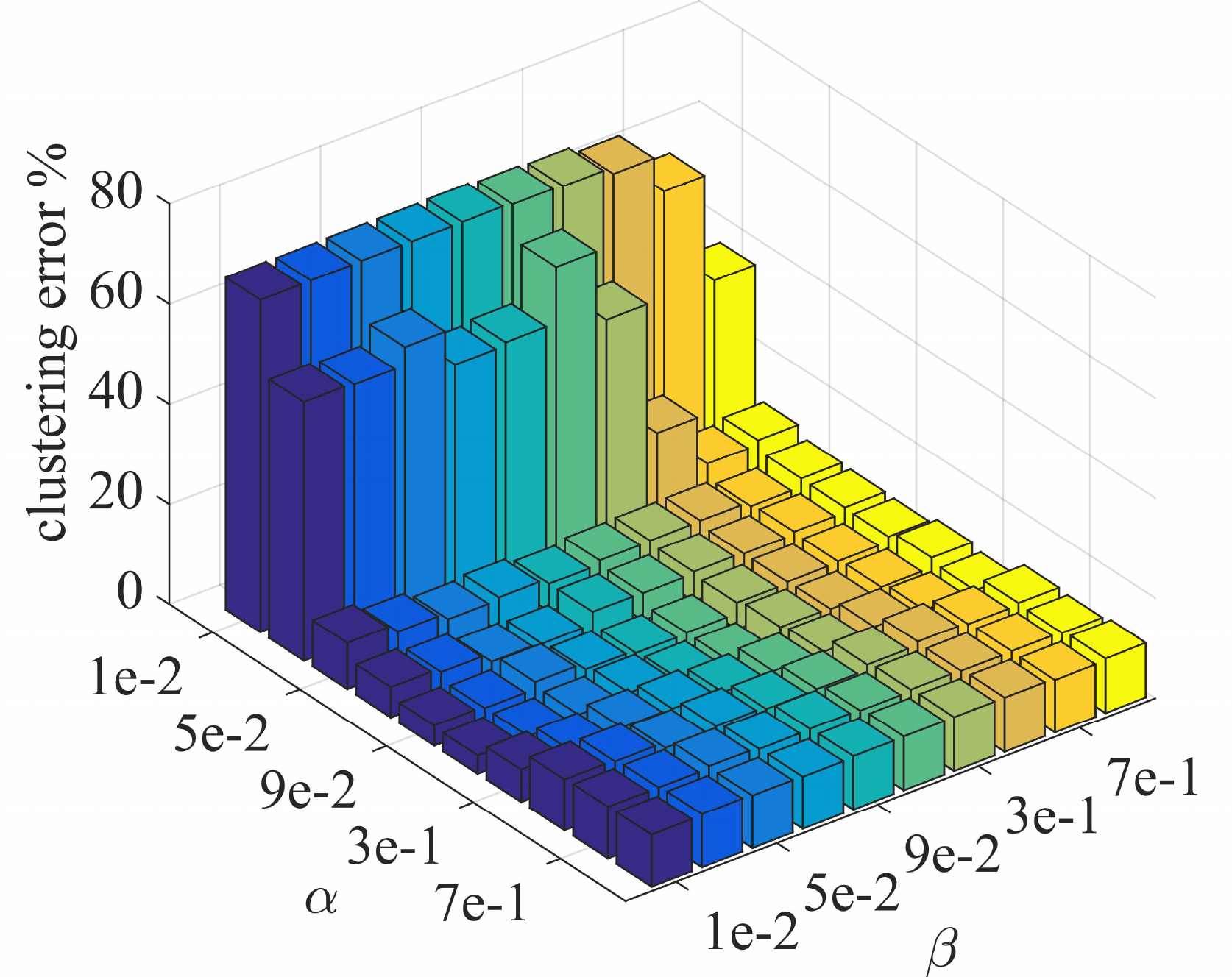}\label{Fig:Para_alphabetamean}}
  \subfigure[]{\includegraphics[width=0.3\textwidth]{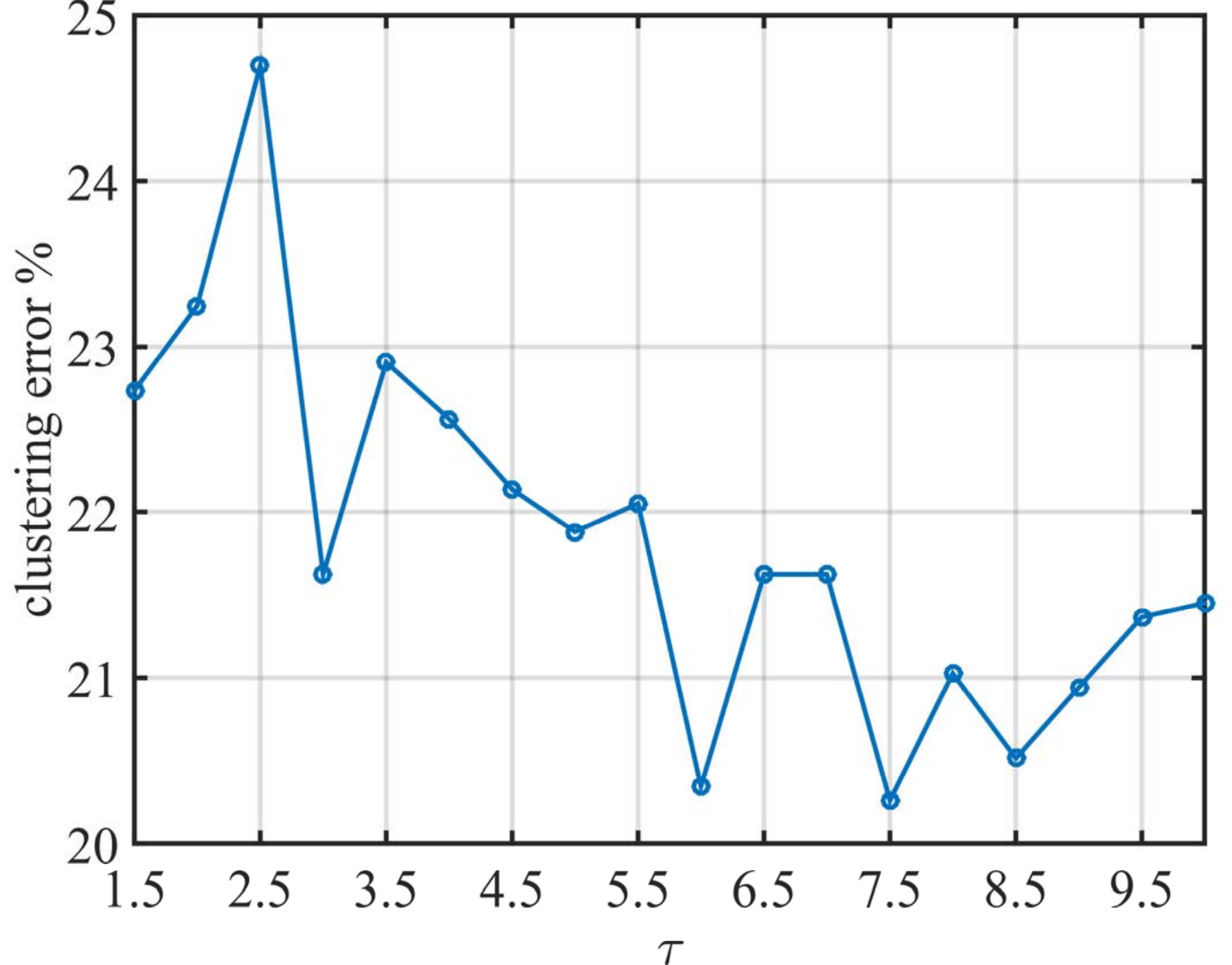}\label{Fig:Para_tau}}
    \subfigure[]{\includegraphics[width=0.3\textwidth]{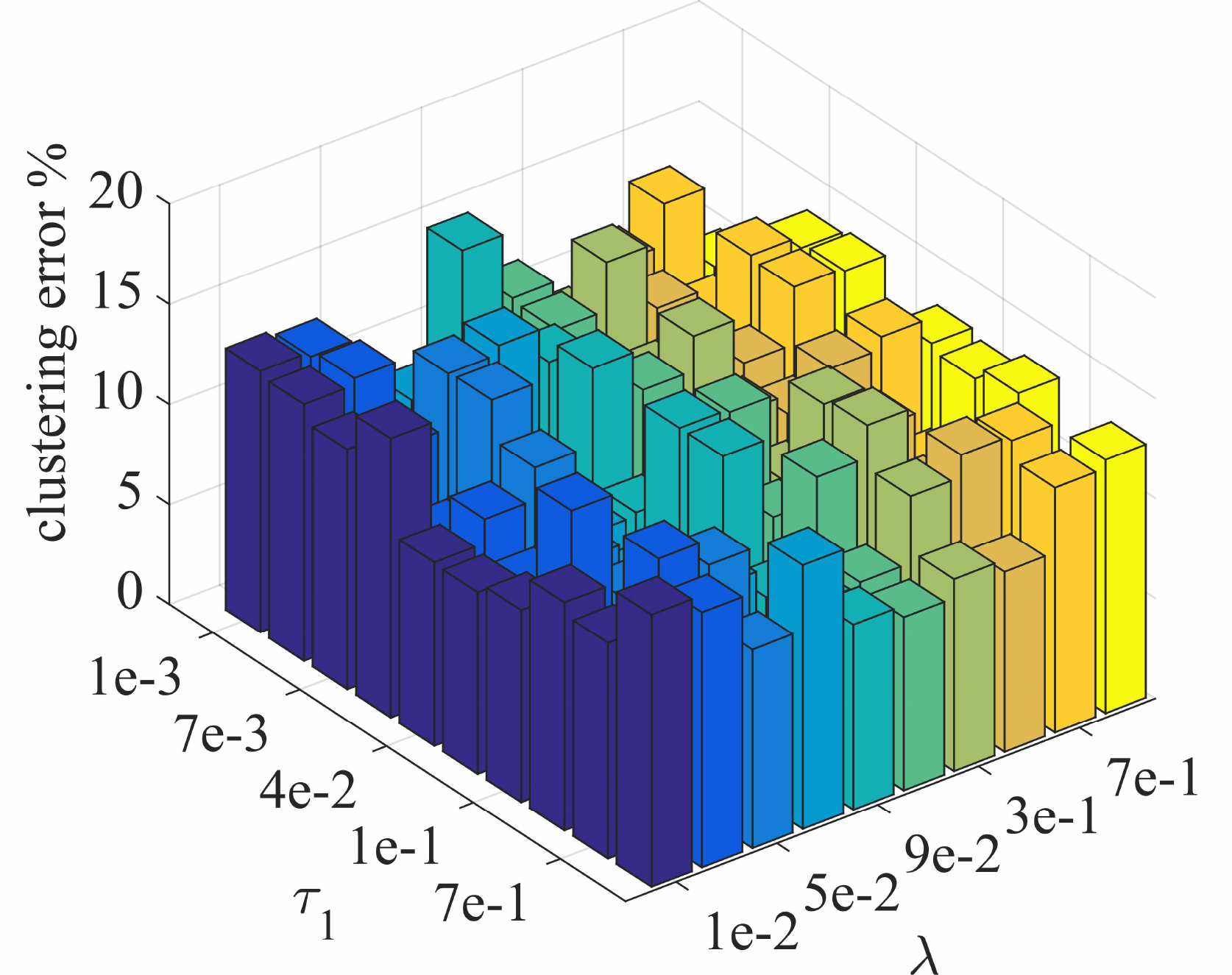}\label{para_lambdatau1}}
  \caption{Clustering error v.s. different regularization parameters. (a). $\alpha$ and $\beta$ in FSSC. (b). $\tau$ in FSSC. (c). $\lambda$ and $\tau_1$ in DTL.}\label{Fig:Para}
\end{figure*}

\subsection{Framework Validation}
\par In this subsection, several illustrative experiments will be designed to demonstrate the effectiveness of our proposed framework and make it more convinced. In the first place, we will evaluate the behaviour of our designed fuzzy label matrix $\mathbf{Q}$ in FSSC and DTL. Next, discrimination of $\mathbf{A}$ will be investigated visually and quantitively.
\subsubsection{Fuzzy Membership Evaluation}
\par In order to demonstrate the performance of the fuzzy label matrix, a deterministic binary matrix will be simultaneously computed with the standard spectral clustering algorithm while other procedures in DTL-FSSC will remain unchanged. For the sake of declaration, we select and illustrate the typical clustering error curves of 5-classes in Figs. \ref{Fig:ValidateFuzzy}, where $p=N$ on EYB and $p=10K$ on BAD are evaluated, respectively. In Fig. \ref{Fig:FuzzyvsDeterministic}, two types of label matrices will both decrease the clustering errors as iterations proceeding and achieve much better performance than the initialized result of SSC. However, the noteworthy distinction will focus on the speed of error convergence. It is obviously noticed that fuzzy one can reach a much rapid decrease and converge within 10 iterations while the binary matrix will be slower. In Fig. \ref{Fig:FuzzyvsDeterministic_10K}, we illustrate another typical phenomenon, where clustering error with fuzzy membership will suddenly increase but it can gradually make a refinement as iterations proceeding to converge to a satisfactory result at last. On the contrary, binary membership matrix seems to be weak from this perspective. These two typical exemplars demonstrate that our fuzzy label membership will be more robust than a binary one, even if we empirically find that fuzzy membership cannot always achieve a lower clustering error than the binary one.
 \begin{figure}
  \centering
   \subfigure[]{\includegraphics[width=0.24\textwidth]{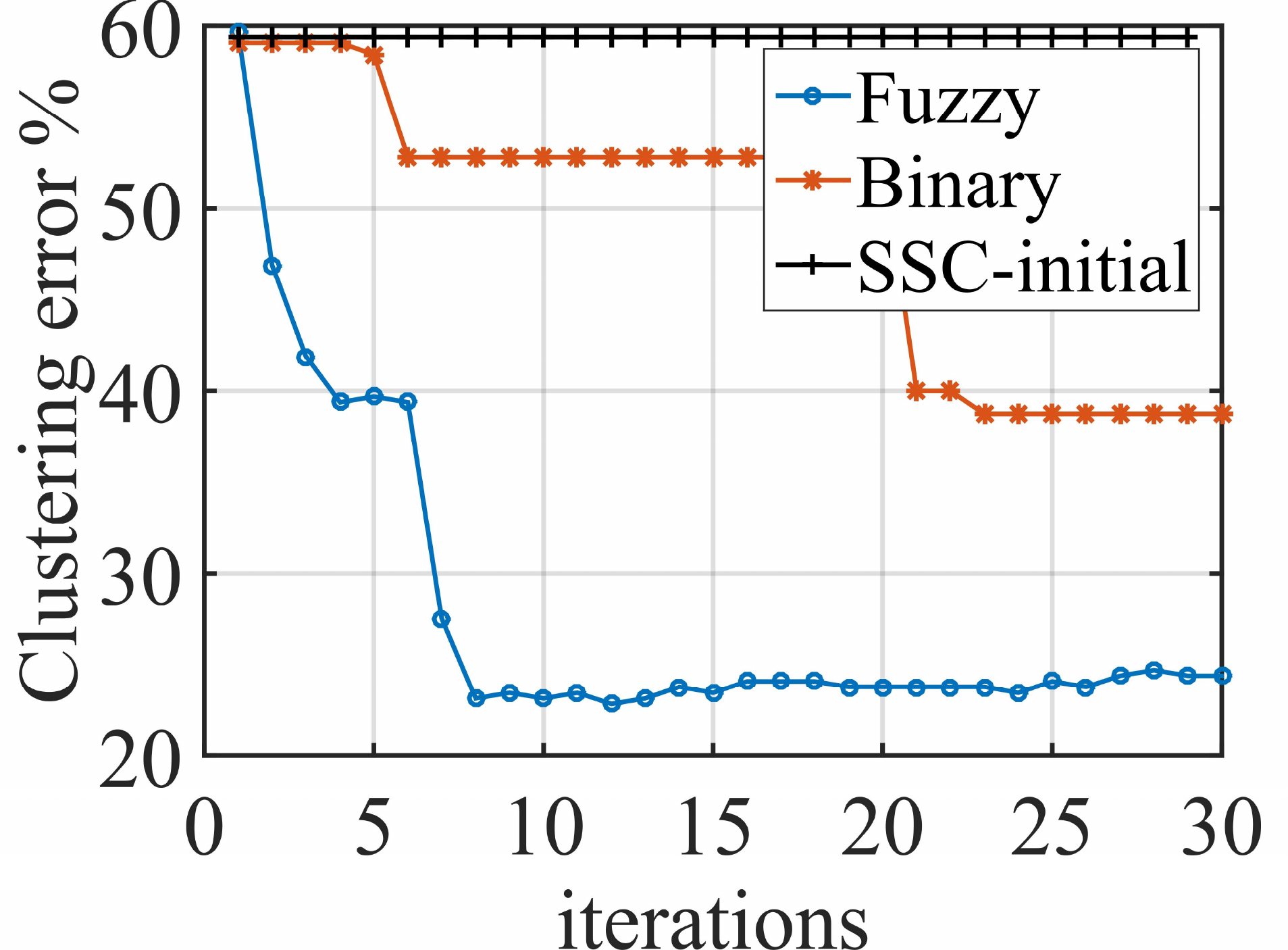}\label{Fig:FuzzyvsDeterministic}}
  \subfigure[]{\includegraphics[width=0.24\textwidth]{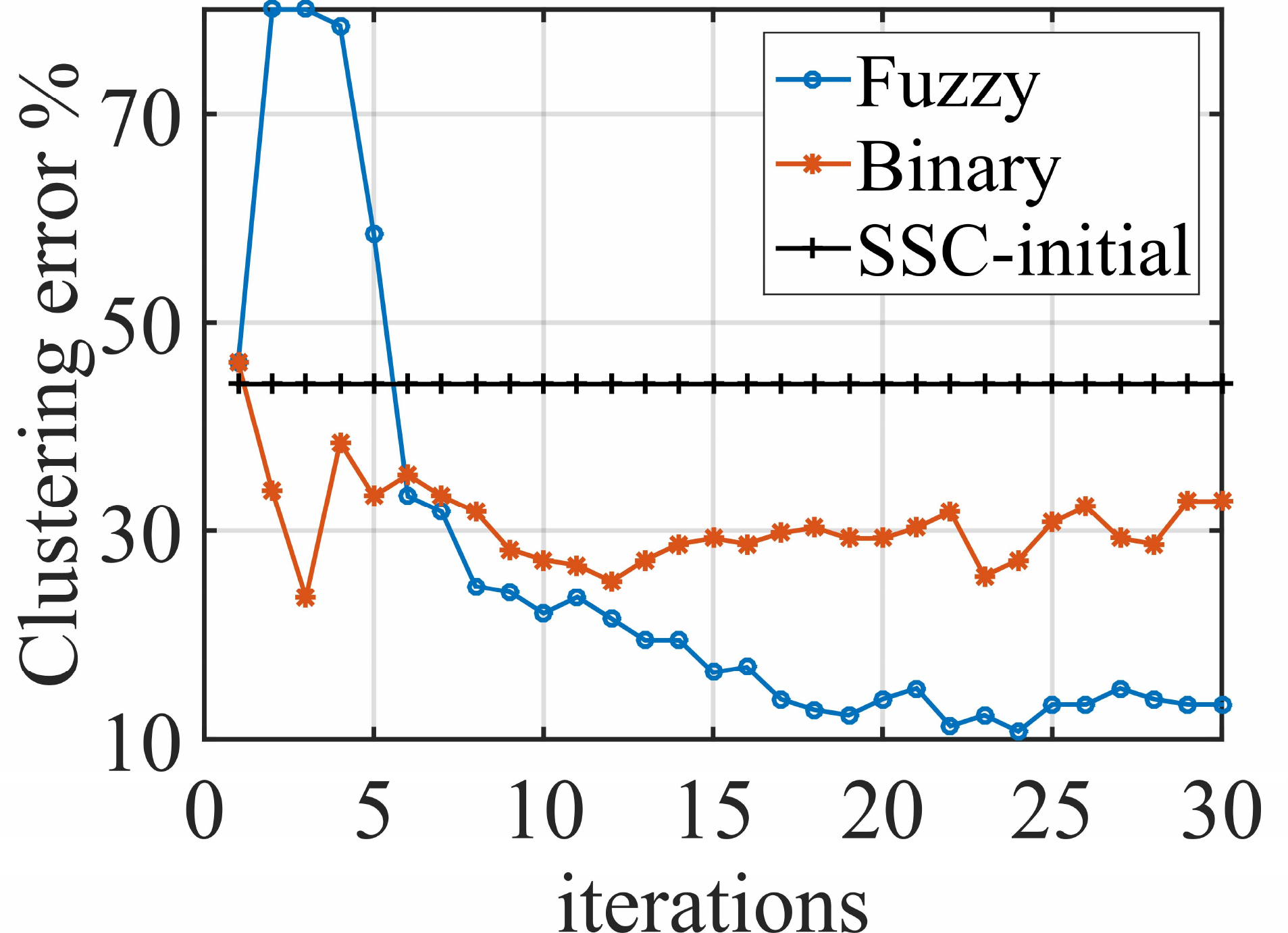}\label{Fig:FuzzyvsDeterministic_10K}}
  \caption{Clustering performances of fuzzy and deterministic membership in different dimensional transformed subspaces. (a). $p=N$ on EYB (b). $p=10K$ on BAD.}\label{Fig:ValidateFuzzy}
\end{figure}

 \subsubsection{Discrimination Evaluation}
Next, we will visualize the transformed projections to check their separations. To this end, Figs. \ref{Fig:scatterplot} plot the scatters of transformed vectors $\mathbf{AX}$ embedded in a two dimensional plane with PCA for visualization, where 3-classes samples in EYB and $3,~5,~10$ digit from BAD are respectively clustered in a $N$-dimensional feature domain. From the visual plots, we can clearly observe that separability between each subspace indeed increases as iteration proceeding, especially for EYB.
 \begin{figure}
  \centering
   \subfigure[]{\includegraphics[width=0.15\textwidth]{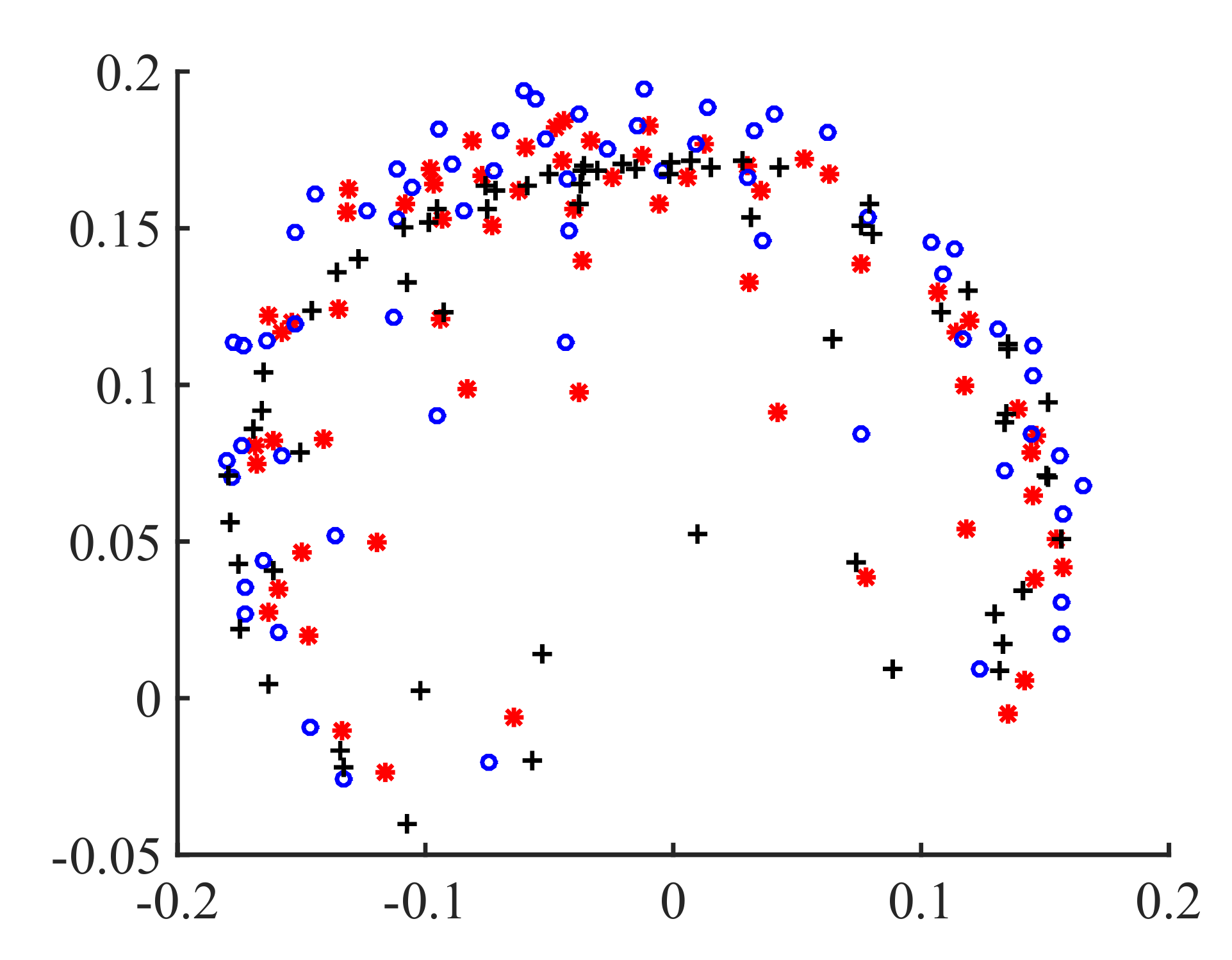}\label{Fig:Primary2DfaceScatter}}
  \subfigure[]{\includegraphics[width=0.15\textwidth]{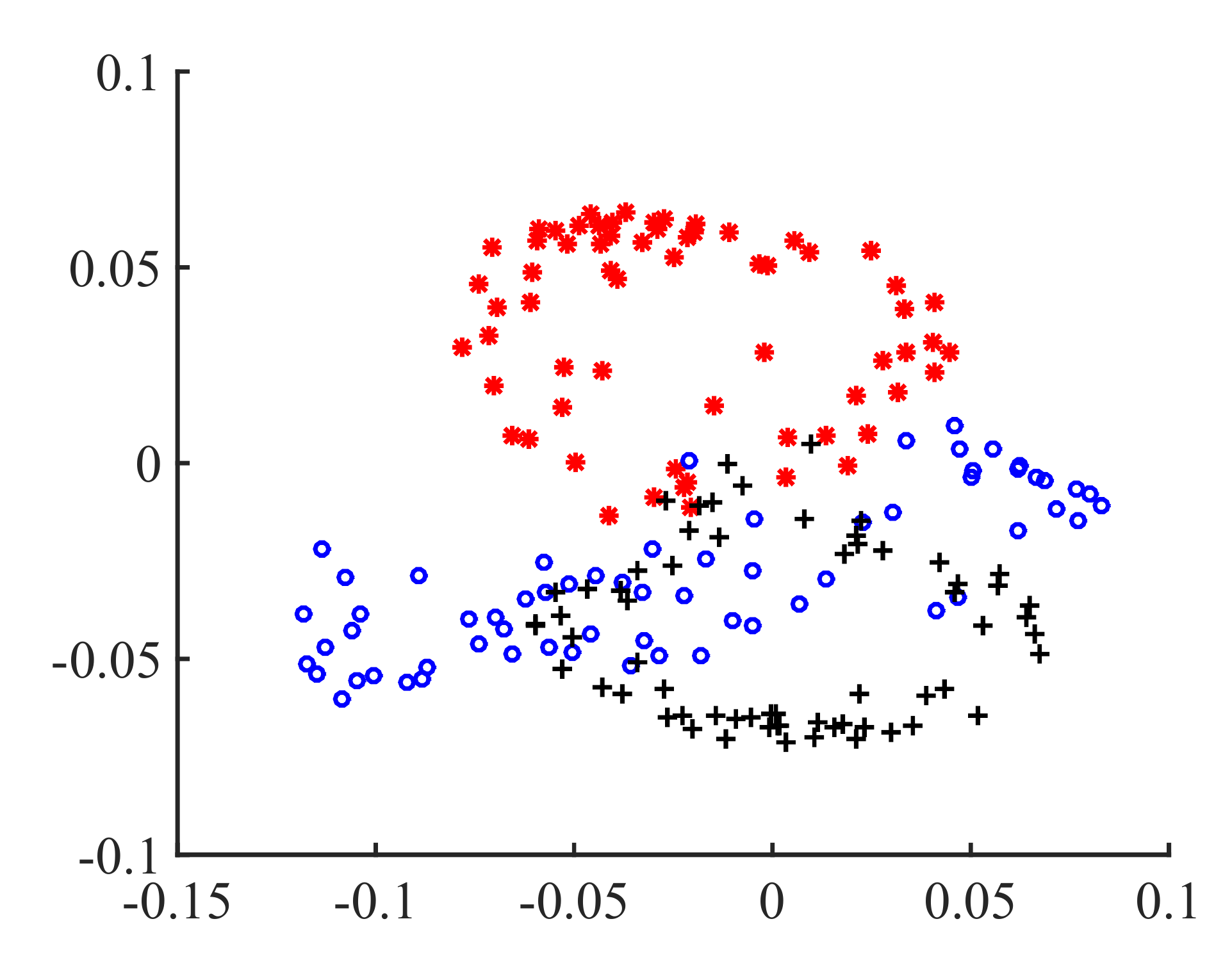}\label{Fig:Primary2DfaceScatterafter2nditeration}}
  \subfigure[]{\includegraphics[width=0.15\textwidth]{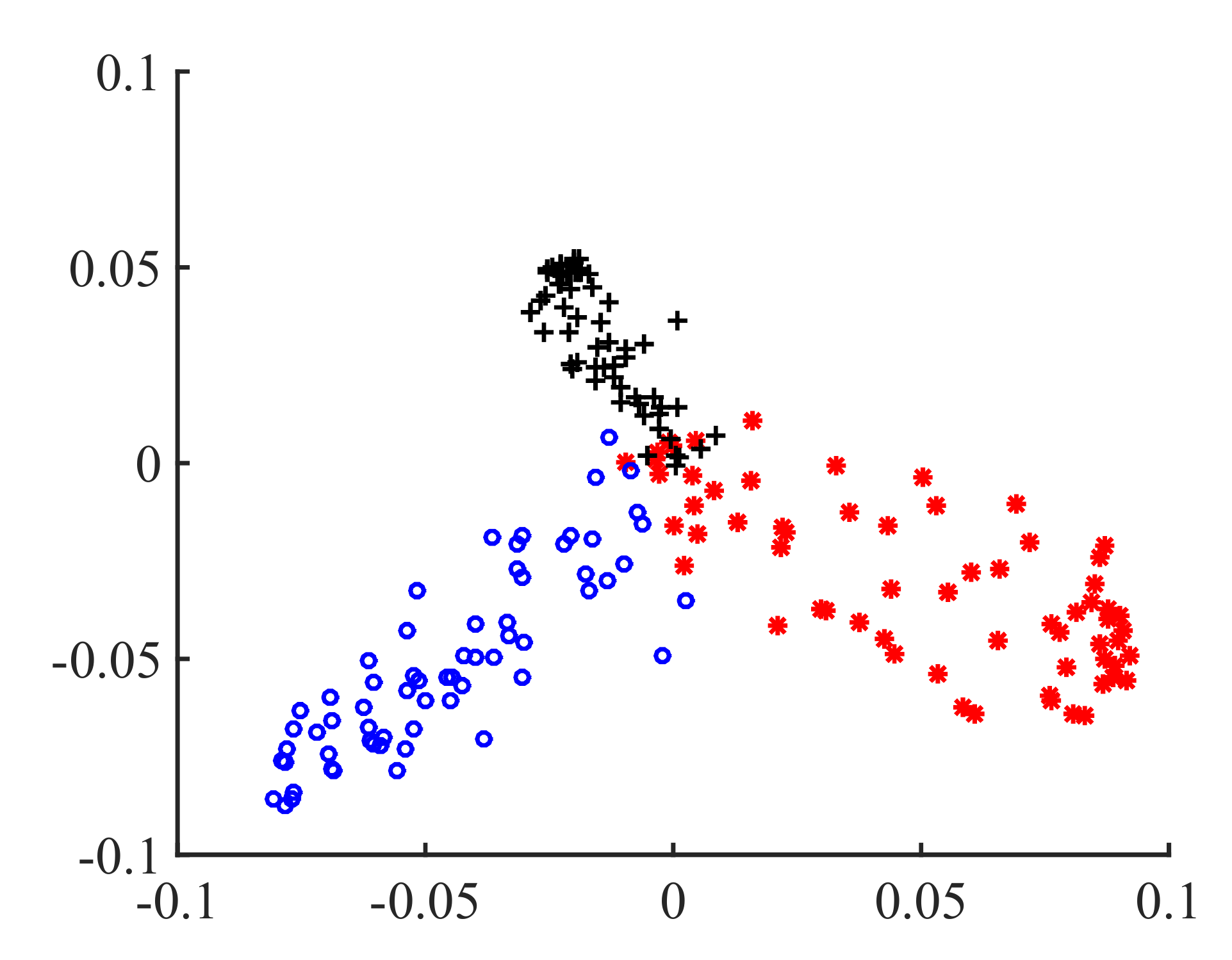}\label{Fig:Primary2DfaceScatterafter4thiteration}}
     \subfigure[]{\includegraphics[width=0.15\textwidth]{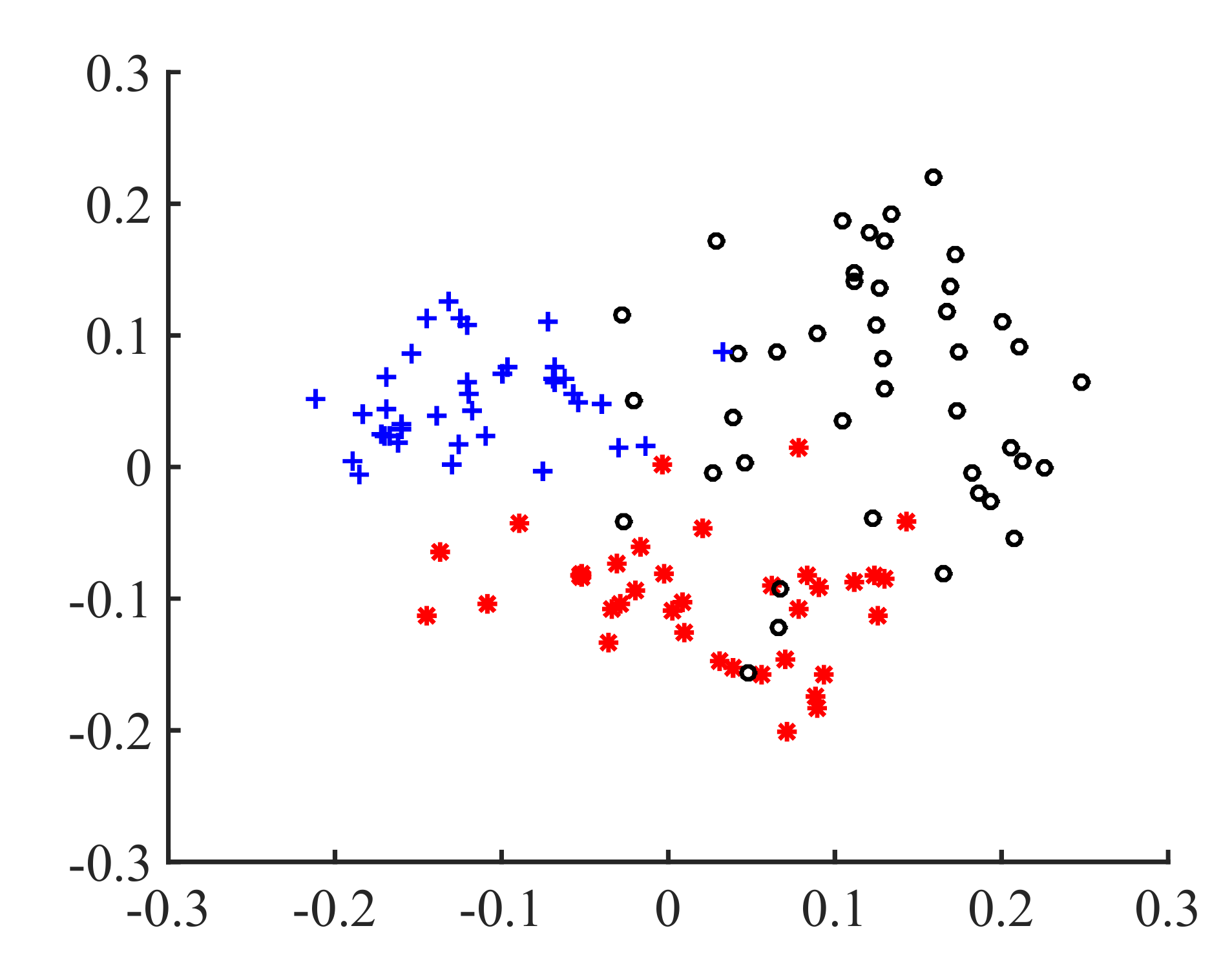}\label{Fig:Primary2DDigitScatter}}
  \subfigure[]{\includegraphics[width=0.15\textwidth]{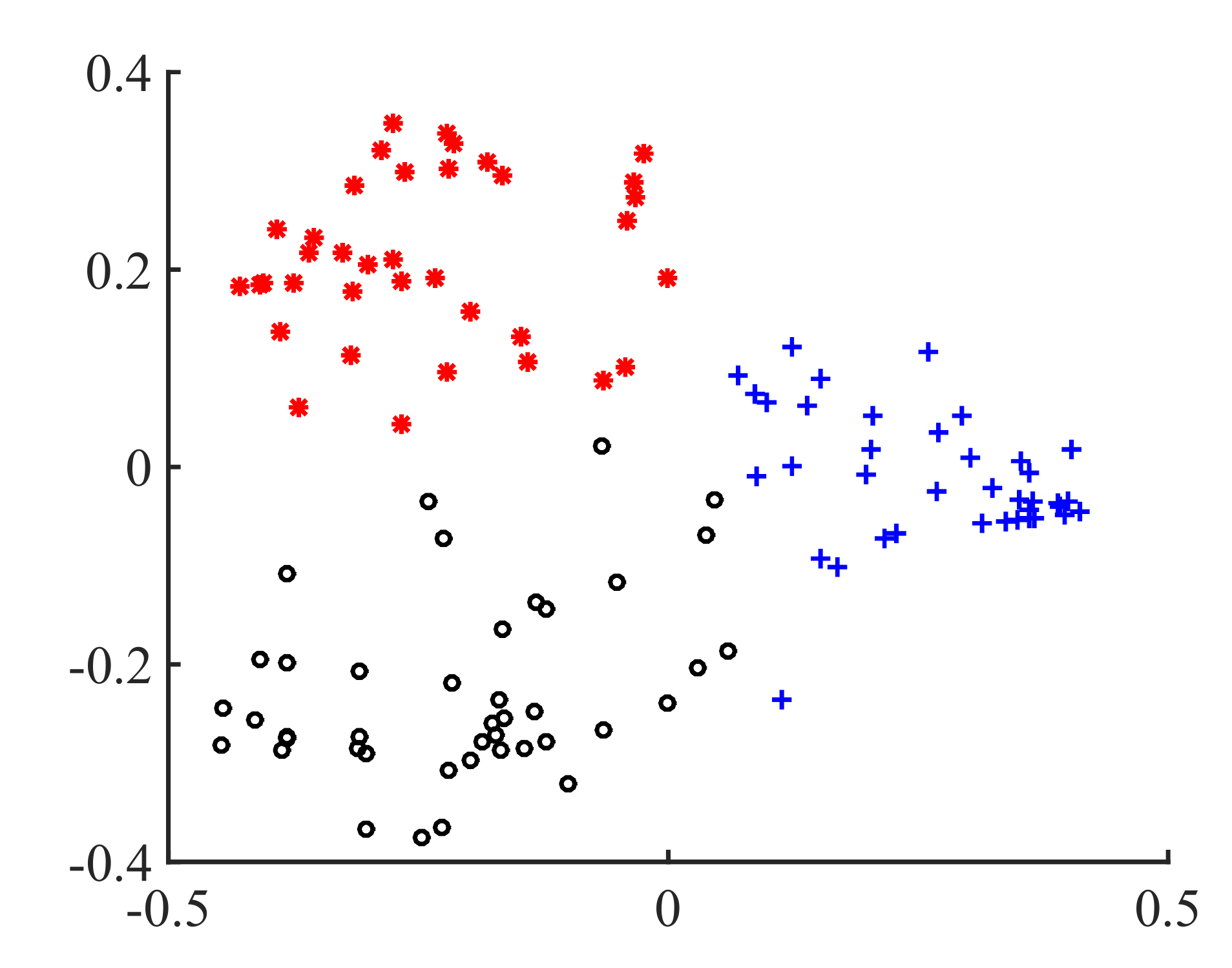}\label{Fig:Primary2DDigitScatterafter2nditeration}}
  \subfigure[]{\includegraphics[width=0.15\textwidth]{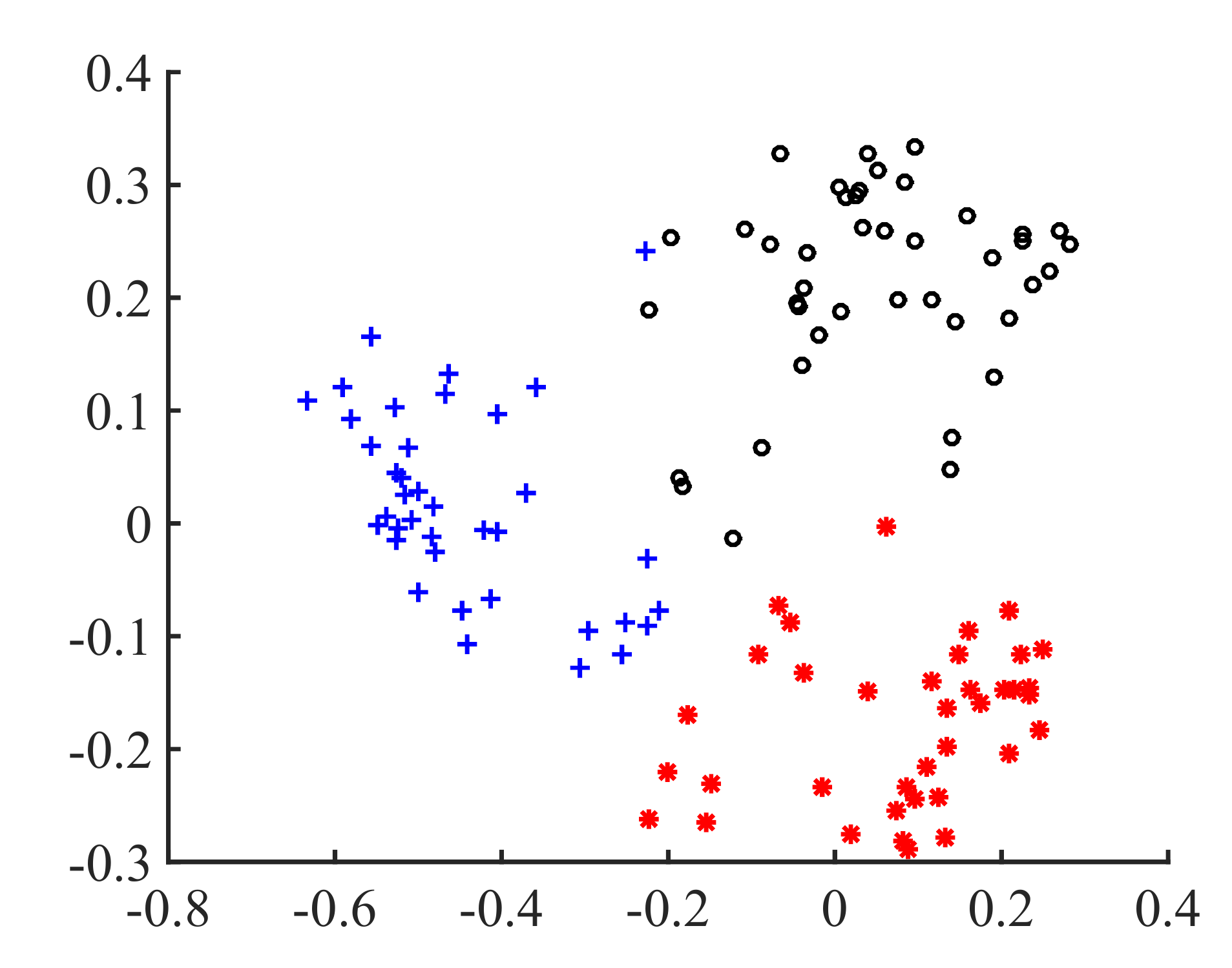}\label{Fig:Primary2DDigitScatterafter4thiteration}}
  \caption{Scatters plot in the two-dimensional plane. (a)-(c): EYB. (d)-(f): BAD (Digits 3,5,10). (a),(d) projected original data. (b),(e). projected transformed data after 2nd iteration. (c),(f). projected transformed data after 4th iteration.}\label{Fig:scatterplot}
\end{figure}
\par To quantify its performance, we compute the smallest principal angle between the transformed subspace $\mathcal{S}_i$ and $\mathcal{S}_j$ defined as following \cite{Elhamifar2013}
\begin{equation}\label{SmallestAngle}
  \theta_{i,j}=\min_{\mathbf{x}_i\in\mathcal{S}_i,\mathbf{x}_j\in\mathcal{S}_j}\arccos\frac{\mathbf{u}^\mathrm{T}\mathbf{v}}{\|\mathbf{u}\|_2\|\mathbf{v}\|_2},~\mathbf{u}=\mathbf{Ax}_i,\mathbf{v}=\mathbf{Ax}_j \end{equation}
where $\mathbf{A}$ is the learned transformation matrix and $\mathbf{x}_i$ is a sample in $i$-th subspace. Roughly speaking, a larger principal angle $\theta_{i,j}$ is, the more separated of $\mathcal{S}_i$ and $\mathcal{S}_j$ will be and it will be easier for clustering. The quantities of the smallest principle angles are depicted in Figs. \ref{Fig:principalangleplot}, where samples from 5-classes of EYB are leveraged. Two conclusions can be induced from Figs. \ref{Fig:principalangleplot} as following: 1). Our method will significantly increase the principle angle between each pair of subspace, yielding more separated transformation subspaces; 2). Increasing the smallest principal angle will be benefit for clustering.

 \begin{figure*}
  \centering
   \subfigure[]{\includegraphics[width=0.26\textwidth]{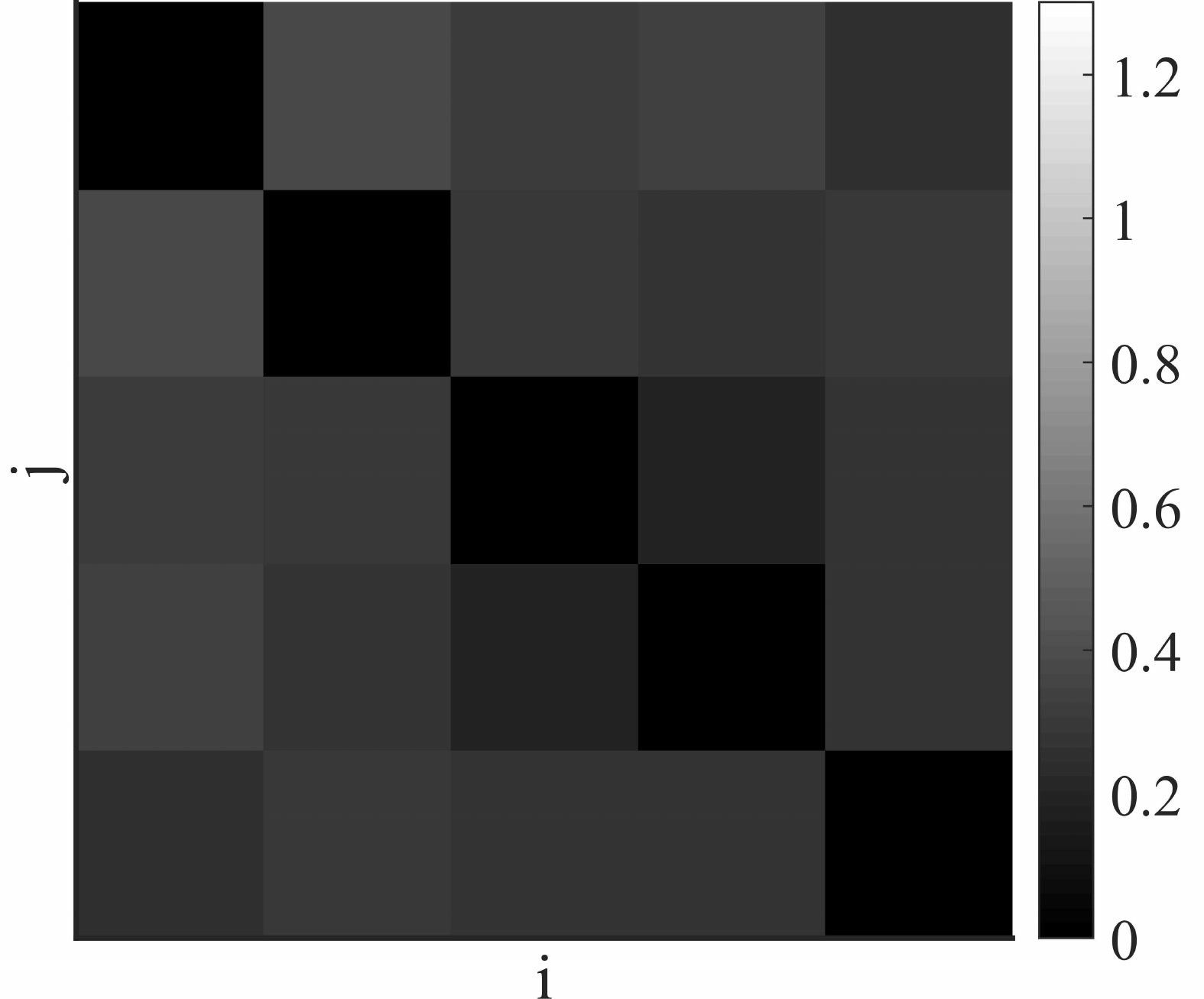}\label{Fig:PrincipalAngleFace}}
  \subfigure[]{\includegraphics[width=0.26\textwidth]{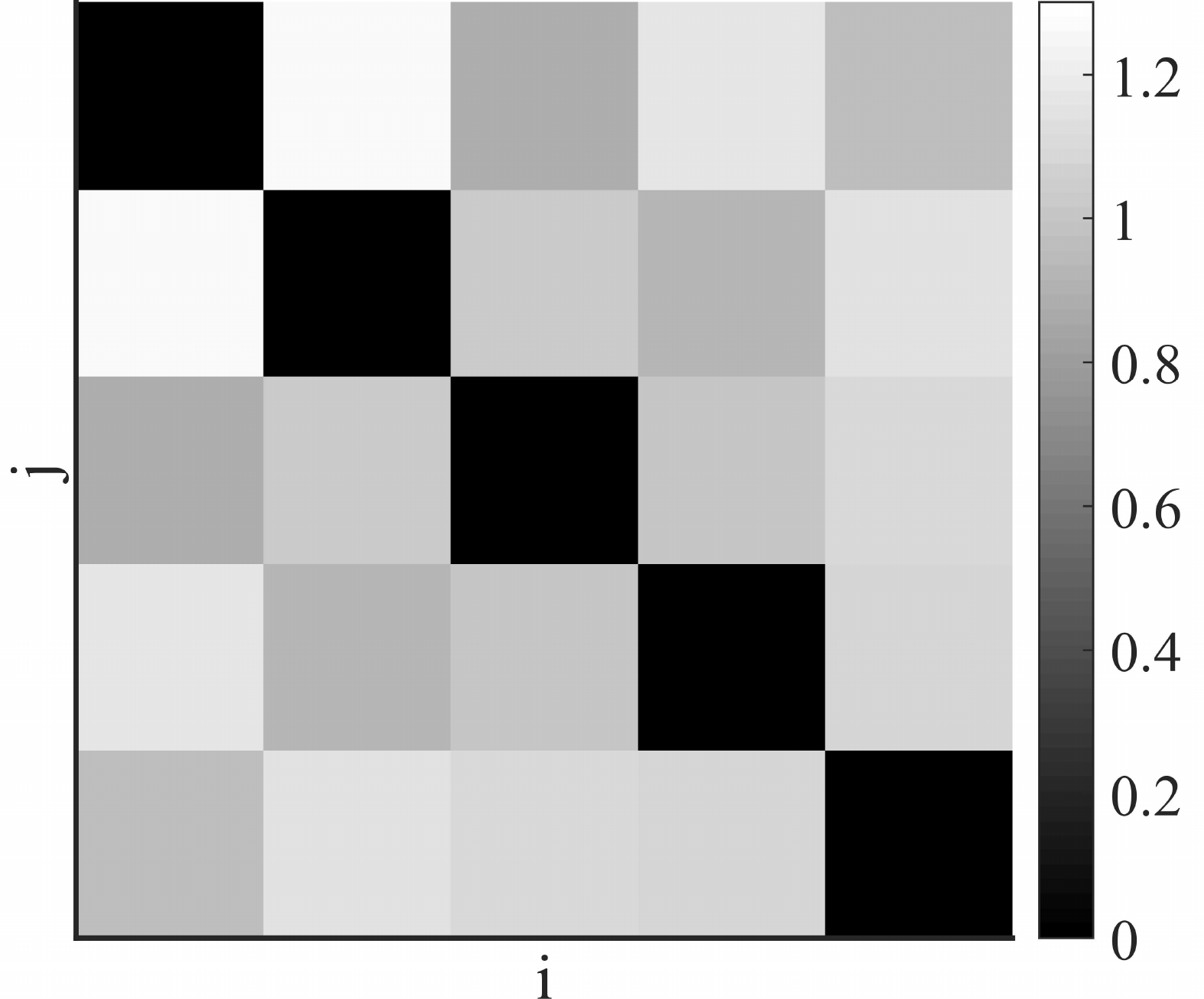}\label{Fig:PrincipalAngleFace_Iteration10}}
  \subfigure[]{\includegraphics[width=0.3\textwidth]{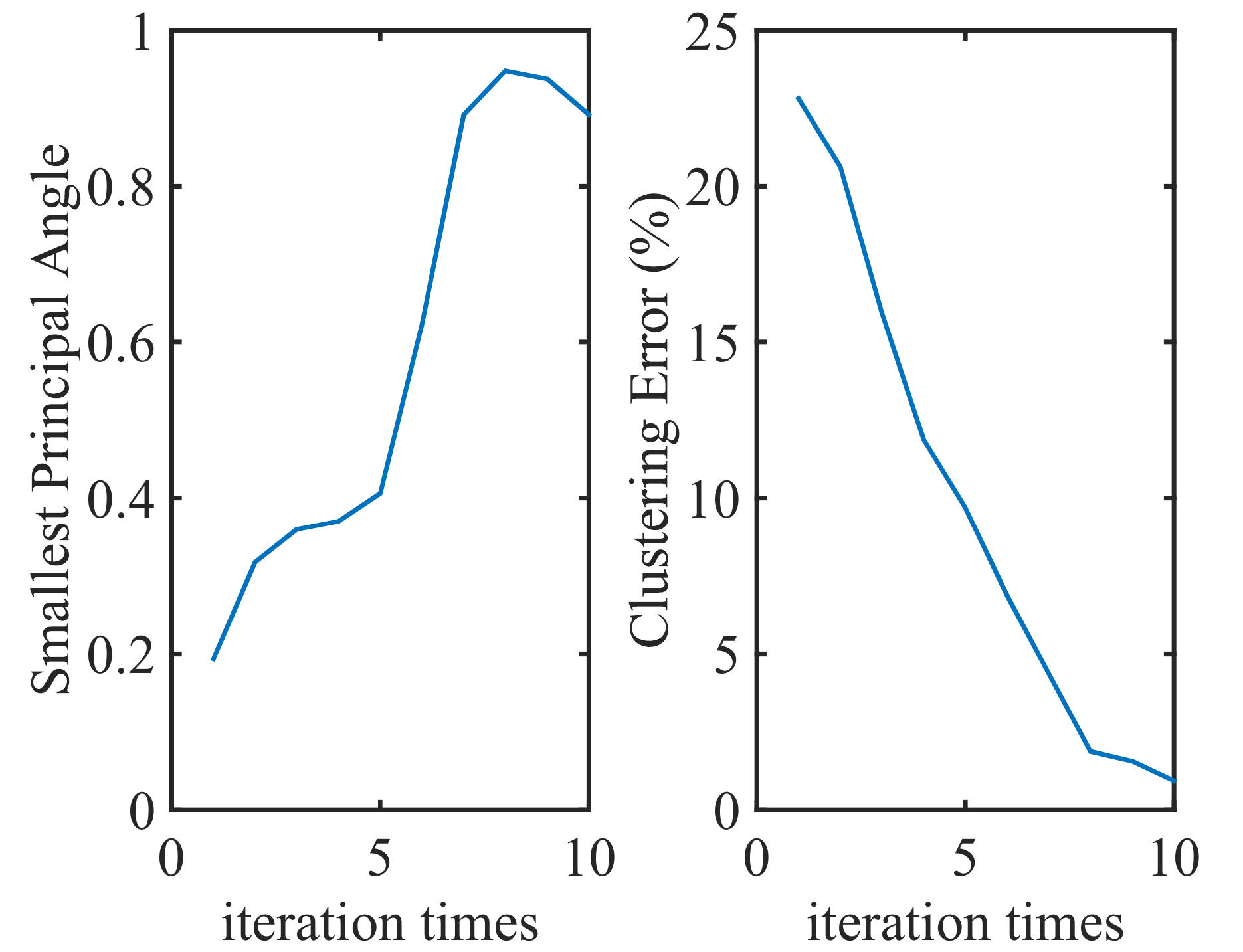}\label{Fig:Face_smallprincipalanglevsmissrate}}
  \caption{Illustration of principal angles $\theta_{i,j}$ and smallest angles v.s. clustering error. (a). principal angles of the original data. (b). principal angles of the transformed data. (c). smallest principal angle v.s. clustering error.}\label{Fig:principalangleplot}
\end{figure*}

\subsection{Frameworks Comparison}
In this subsection, the superiorities of the proposed DTL-FSSC will be demonstrated by comparing with the other state-of-the-art SC approaches on three typical benchmark databases, including motion segmentation, handwritten image clustering, face clustering.
\subsubsection{Motion Segmentation} The task of motion segmentation is one of the typical subspace segmentation applications which aims to segment a sequence of video into multiple spatiotemporal regions representing different types of motions in the scene.  Let $\{\mathbf{f}_{i,j}\in\mathbb{R}^2\}_{j=1}^N$ be a set of $N$ tracked feature points in $i$-th frame of the video and $i=1,\dots,F$ be the frame index. The $j$-th so-called feature trajectory vector $\mathbf{x}_j$ will be obtained by stacking $\mathbf{f}_{i,j}$ from all frames as
\begin{equation}\label{Equ:featuretrajectory}
  \mathbf{x}_j=[\mathbf{f}_{1,j};\cdots ;\mathbf{f}_{F,j}]\in\mathbb{R}^{2F}
\end{equation}
Motion segmentation will separate all $\{\mathbf{x}_j\}_{j=1}^N$ based on their underlying motions and it has been shown that these feature vectors from the same motion will lie in an affine subspace of at most 4 dimension \cite{Boult1991}. As a consequence, different type of motions will reside in different subspace so that the task of motion segmentation will be handled by clustering of the feature vectors in a union of affine subspaces. Hopkins 155 dataset is one of the most popular benchmarks for this application which contains 155 video sequences of two or three motions with checkerboard, traffic and articulated contents \cite{Tron2007}. On average, two motions of 120 sequences with about $N=266$ feature points and $F=30$ frames are collected while $35$ sequence of three motions contain about $N=398$ trajectories and $F=29$ frames. Two sample frames from Hopkins 155 dataset are shown in Figs. \ref{Fig:Hopkins155}.
 \begin{figure}
  \centering
   \subfigure[]{\includegraphics[width=0.24\textwidth]{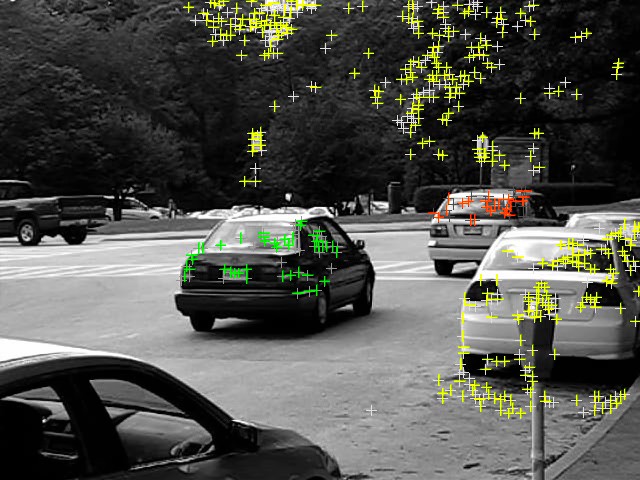}\label{Fig:preview_car}}
  \subfigure[]{\includegraphics[width=0.24\textwidth]{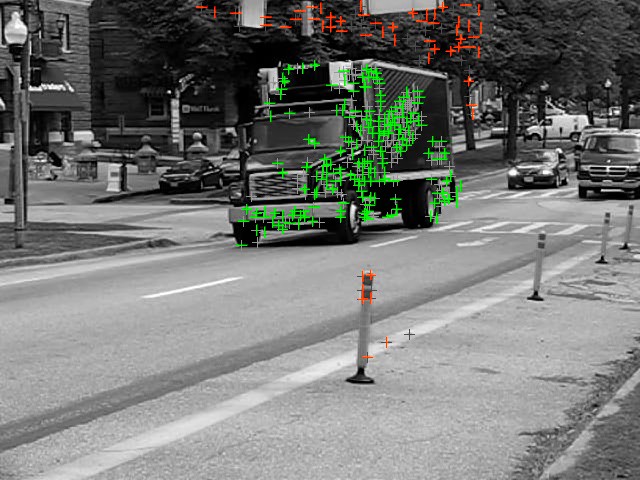}\label{Fig:preview_truck}}
  \caption{Sample frames from Hopkins 155 dataset. (a). car. (b). truck.}\label{Fig:Hopkins155}
\end{figure}
\par In the first group of experiments, we will consider clustering in a subspace with primary dimension, namely $p=2F$, where DTL-FSSC will be compared with the baseline SSC \cite{Elhamifar2013}, LRR \cite{Liu2013}, LRT* \cite{Qiu2014}, SCC \cite{Chen2009} and a typical motion segmentation algorithm, namely LSA \cite{Rao2010Motion}. The comparison results of mean and median values are shown in Table. \ref{Table:MS2Fdimensionresult} and the best performance for each item here and after will be denoted in boldface. For fair comparison, the results listed are obtained from the corresponding papers or reproduced according to the optimal setting declared in the corresponding papers. The postprocessing procedures of LRR and SSC are also taken into account which will greatly improve their performances. In standard LRT, the authors extra involved a RPCA step in each iteration to get rid of some outliers, however, this algorithm in our simulations is left out only to highlight the importance of exploiting $\mathbf{Z}$ and fuzzy $\mathbf{Q}$ during operator learning phase. Therefore, our reported performances will be slightly different from \cite{Qiu2014} denoted by LRT*, but it will not essentially influence our comparison. For our DTL-FSSC, the detailed regularization parameters are configured as following. $\alpha=0.03$, $\beta=0.5$, $\tau=4$, $\lambda=0.05$ and $\tau_1=1$.
\par From the results, we can make a conclusion that DTL-FSSC outperforms the all competitive algorithms in the case of two subspaces clustering, where only $1.8\%$ average error can be achieved. Compared with the baseline SSC providing the initial $\mathbf{Z}$, our proposed DTL-FSSC can further decrease the error by $0.34\%$. For 3 subspaces situation, although LRR outperforms the proposed method in terms of average error rate, DTL-FSSC can still bring out the least error rate in median, which means a more stable performance than others. In this situation, DTL-FSSC decreases the error rate from $5.27\%$ to $4.20\%$ demonstrating the superiorities  of  our framework. Additionally, the performances of LRT* are not satisfactory than others on this database, which can be inferred that their discriminative regularization will not robust to the outliers when we leave out the RPCA refinement.
\begin{table}
\caption{Clustering Error (\%) on the Hopkins 155 Dataset in the 2F-Dimensional Subspace\label{Table:MS2Fdimensionresult}}\centering
\begin{tabular}{|c|c|c|c|c|c|c|c|}
\hline
K& Alg & LSA & SCC & LRR & LRT* & SSC & DTL-FSSC\tabularnewline
\hline
\hline
\multirow{2}{*}{2} & mean & 4.23 & 2.89 & 2.21 & 9.61 & 2.14 & \textbf{1.80}\tabularnewline
\cline{2-8}
 & median & 0.56 & \textbf{0.00} & \textbf{0.00} & 2.92 & \textbf{0.00} & \textbf{0.00}\tabularnewline
\hline
\multirow{2}{*}{3} & mean & 7.02 & 8.25 & \textbf{3.84} &17.90 & 5.27 & {4.20}\tabularnewline
\cline{2-8}
 & median & 1.45 & 0.24 & 1.43 & 18.13 & 0.56 &\textbf{0.21} \tabularnewline
\hline

\end{tabular}
\end{table}

\par In the next group of experiments, clustering performance will be evaluated in some low dimensional feature domains with different dimensions, namely $p=\{2K,~4K,~6K,~8K,10K\}$. Except for the algorithms compared above, the performances of two baseline methods, namely LS3C and NLS3C equipped with polynomial kernel will be further concerned. In particular, for SSC and LRR, we will firstly construct a random matrix $\mathbf{A}$ drawn from normal distribution, and then perform clustering with the projections $\mathbf{A}\mathbf{X}$. We plot the comparison results in Figs. \ref{Fig:Hopkins155_result}. Inspecting the figures, we can conclude that DTL-FSSC will outperform all compared algorithms in most cases and the clustering errors obviously decrease. It is more worth noting that the results of our algorithms are relatively robust to the variation of dimensions, which will mainly own to the structure preservation in the feature domain. Seeing from the results that the compared benchmark algorithms of LS3C and NLS3C do not appear superior performance because no discriminative strategies have been explicitly involved in their frameworks. In summary, the proposed DTL-FSSC can improve the performance of its baseline SSC and bring about at least competitive results than other state-of-the-arts in different situations on Hopkins 155 database.

 \begin{figure}
  \centering
   \subfigure[]{\includegraphics[width=0.24\textwidth]{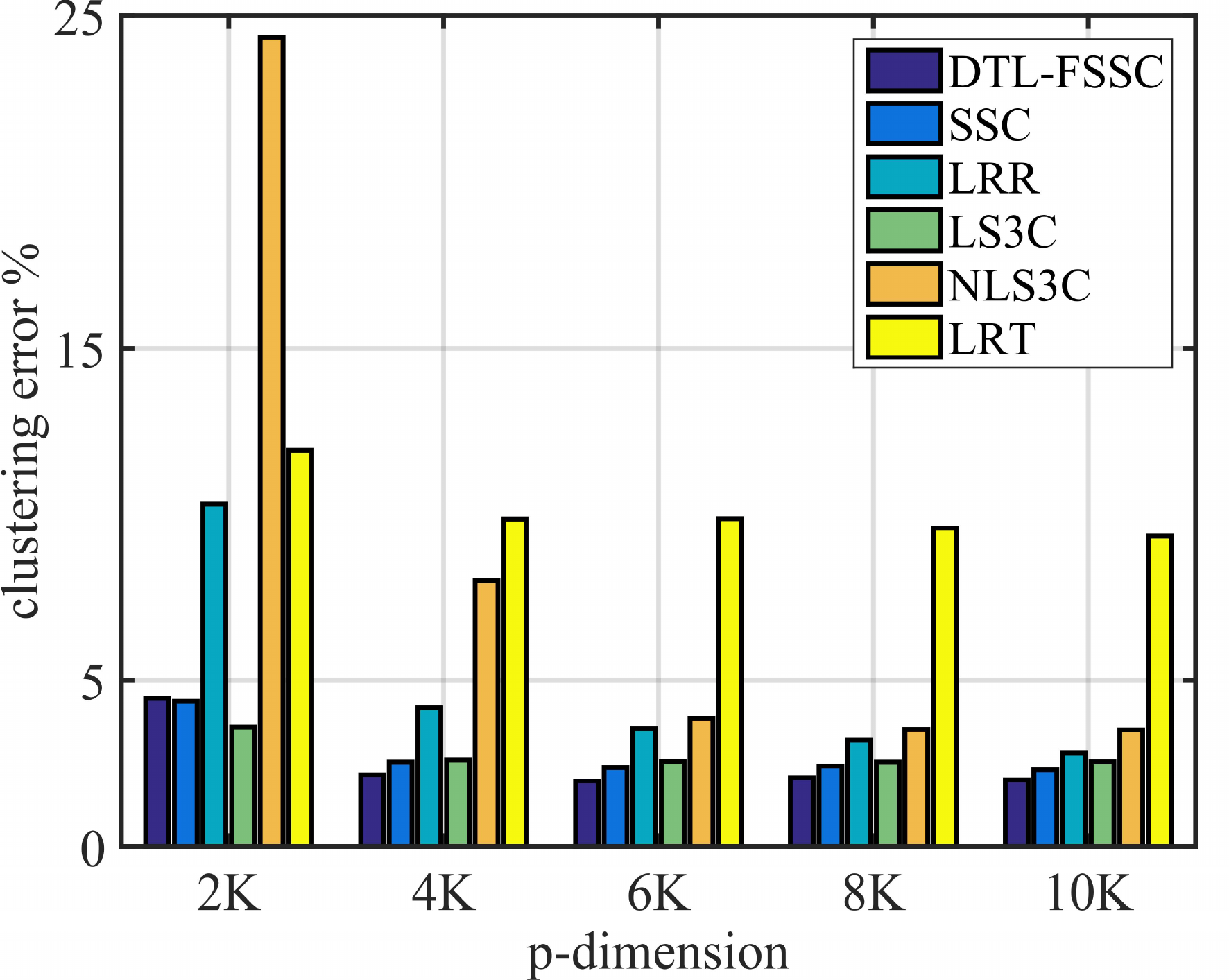}\label{Fig:Hopkins155_varyingdimension_rate}}
  \subfigure[]{\includegraphics[width=0.24\textwidth]{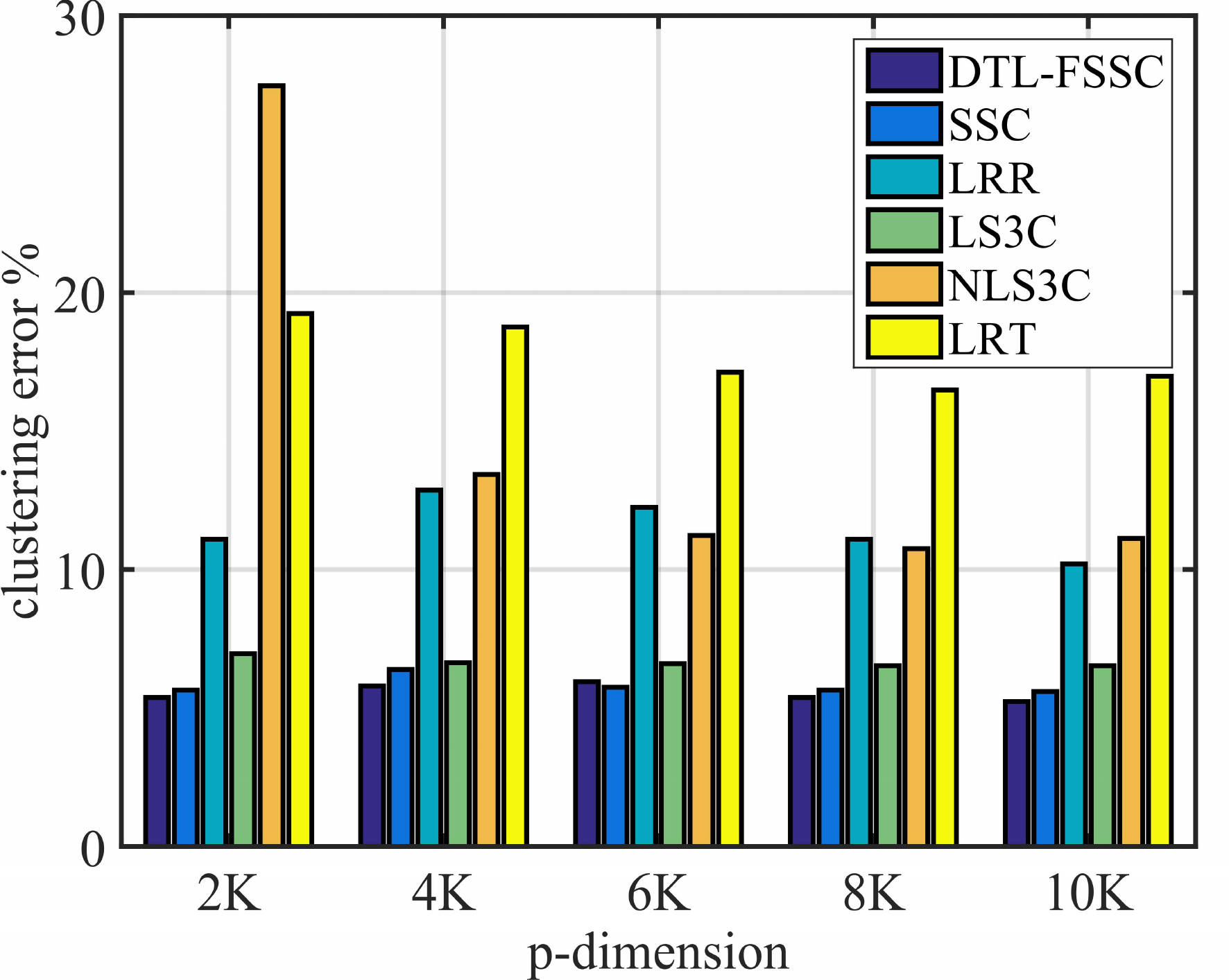}\label{Fig:Hopkins155_varyingdimension_rate3}}
  \caption{Clustering Error v.s. $p$ for different methods. (a). $K=2$. (b). $K=3$.}\label{Fig:Hopkins155_result}
\end{figure}

\subsubsection{Digital Handwritten Clustering}
In this part, we will evaluate the clustering performance on a more difficult database with the increased amount of subspaces as well as the primary dimension, namely BAD. This database comprises 36 classes of handwritten images with the size of $20\times 16$, including $0\thicksim9$ digit and $A\thicksim Z$ alphabet. Each class contains 39 images in total and some sample images are illustrated in Figs. \ref{Fig:sampleimage_digit}.
 \begin{figure}
  \centering
   \subfigure[]{\includegraphics[width=0.24\textwidth]{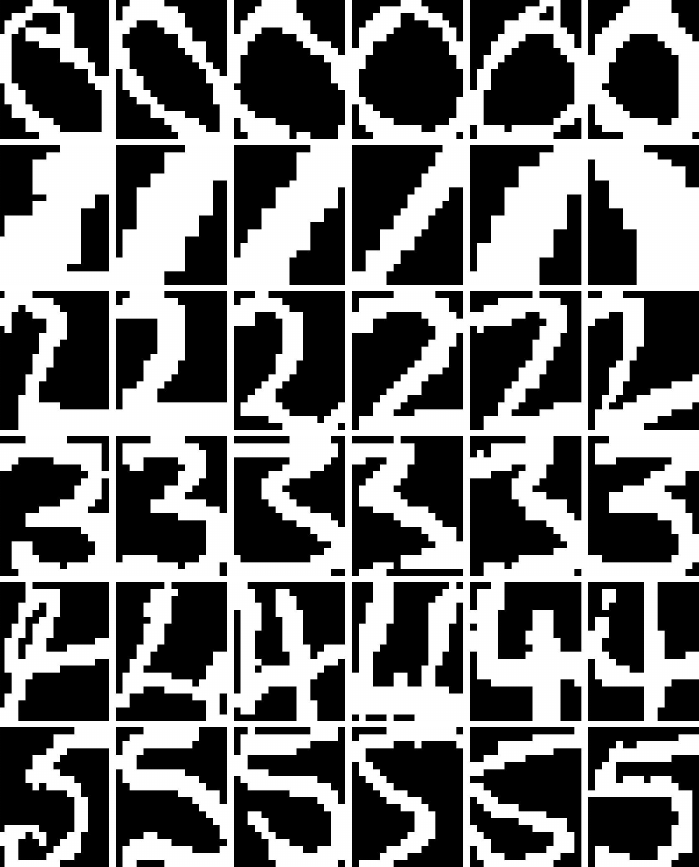}\label{Fig:SampleImage_Dig_number}}
  \subfigure[]{\includegraphics[width=0.24\textwidth]{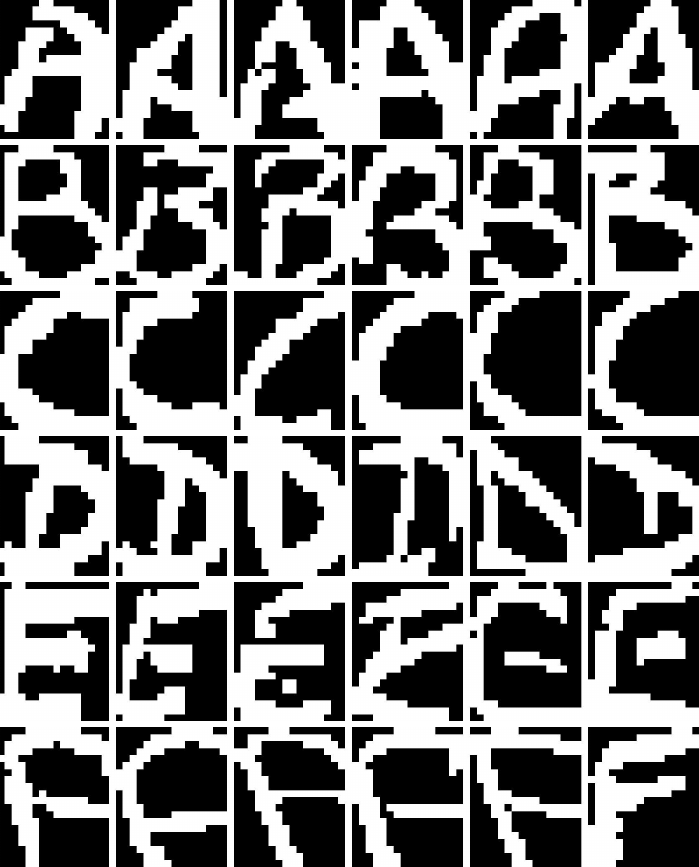}\label{Fig:SampleImage_Dig_alphabet}}
  \caption{Sample images from BAD. (a). digital number $0\sim5$ . (b). alphabet $A\sim F$.}\label{Fig:sampleimage_digit}
\end{figure}

\par To evaluate the effect of the increased number of subspaces, we consider the choice of $K\in\{3,5,10\}$ for clustering. For these choices, we randomly select $K$ out of 36 classes and perform 500, 200 and 100 independent experiments to report the average and median clustering errors, respectively. Moreover, following the previous validation, two different dimensional feature domains will be considered, namely $p=n$ and $p=10K$. The corresponding results are summarized in Table \ref{Table:DigPrimarydimensionresult} and \ref{Table:Dig10Kdimensionresult}, respectively. The parameters in LRR and SSC will be set as $0.7$ and $5$ according to our validations, respectively. The parameters in DTL-FSSC will be set as $\alpha=0.07$, $\beta=0.01$, $\tau=8$, $\lambda=0.07$ and $\tau_1=0.07$.
\par Let us firstly check the clustering performance in the primary dimensional space shown in Table \ref{Table:DigPrimarydimensionresult}. At the first sight, we can observe that the propose method significantly outperforms all compared approaches for all clustering numbers in terms of mean and median clustering error, which clearly demonstrates its superiorities. Note also in the case of $K=5$ that DTL-FSSC decreases the average error from the result of SSC $23.12\%$ to $15.76\%$, which will be about $7\%$ enhancement on average. Considering the median value, $12.82\%$ of DTL-FSSC is more than $10\%$ enhancement than that of SSC. For LRT, its performances on this database can be much better than those of the previous hopkins 155 database. It also reduces the clustering error by a large margin and outperforms other algorithms in the most cases.
\begin{table}
\caption{Clustering Error (\%) on the Binary Alphadigits Dataset in $n$-Dimensional Subspace\label{Table:DigPrimarydimensionresult}}\centering
\begin{tabular}{|c|c|c|c|c|c|c|c|}
\hline
K& Alg & LRT* & LSA&NLS3C & LRR & SSC & DTL-FSSC\tabularnewline
\hline
\hline
\multirow{2}{*}{3} & mean & 10.89 & 22.69&12.85 & 13.17 & 10.91& \textbf{8.40} \tabularnewline
\cline{2-8}
 & med. & 5.98 &22.22& 7.69 & 10.26 & 6.84 & \textbf{4.27}  \tabularnewline
\hline
\multirow{2}{*}{5} & mean & 20.26 &33.81& 22.64 & 23.52 & 22.68 & \textbf{15.76} \tabularnewline
\cline{2-8}
 & med. & 17.44&33.85& 23.08 & 23.59 & 23.08 & \textbf{12.82}  \tabularnewline
\hline
\multirow{2}{*}{10} & mean & 33.61 &42.65& 33.85& 32.62 &35.14 & \textbf{31.27}  \tabularnewline
\cline{2-8}
 & med. & 33.97 &41.28 &34.36 & 32.44 &31.90 & \textbf{31.28} \tabularnewline
\hline
\end{tabular}
\end{table}
\par Considering their performances in the $10K$-dimensional feature domain, although NLS3C achieves better results than DTL-FSSC in the case of $K=3$, DTL-FSSC can also bring about lower average errors in the case of $K=5$ and $K=10$. More importantly, DTL-FSSC merely exploits a linear operator while NLS3C is equipped with a nonlinear kernel, which can further establish the superiority of our framework. Note also from the results of SSC, DTL-FSSC again improves its performance in a large margin. For LRT*, we may observe that its performance drops rapidly by comparing \ref{Table:DigPrimarydimensionresult} and \ref{Table:Dig10Kdimensionresult}. This is mainly caused by information loss of a rank-defective operator on this database.
\begin{table}
\caption{Clustering Error (\%) on the Binary Alphadigits Dataset in $10K$-Dimensional Subspace\label{Table:Dig10Kdimensionresult}}\centering
\begin{tabular}{|c|c|c|c|c|c|c|c|}
\hline
K& Alg & LRT* & LS3C&NLS3C & LRR & SSC & DTL-FSSC\tabularnewline
\hline
\hline
\multirow{2}{*}{3} & mean & 22.69 & 17.37&\textbf{14.11} & 40.01 & 21.81& 19.96 \tabularnewline
\cline{2-8}
 & med. & 22.22 &12.82& \textbf{8.55} &41.03 & 17.95 & 11.97 \tabularnewline
\hline
\multirow{2}{*}{5} & mean & 29.74 &26.75& 25.03 & 38.76 & 29.77 & \textbf{21.62} \tabularnewline
\cline{2-8}
 & med. & 29.19 &24.10& 23.85 & 38.97 & 31.28 & \textbf{20.00}  \tabularnewline
\hline
\multirow{2}{*}{10} & mean & 40.99 &38.81& {34.05}& 48.26 &40.79 & \textbf{34.02}  \tabularnewline
\cline{2-8}
 & med. & 39.87 &38.59 &\textbf{32.56} & 48.21 &38.33 & 33.90 \tabularnewline
\hline
\end{tabular}
\end{table}

\subsubsection{Face Clustering}
\par Finally, we come to validate the proposed framework on EYB database for face clustering, which is much more sophisticated than the previous ones due to high ambient dimension and illumination corruptions. There are 64 images in EYB with the size of $192\times 168$ and the total number of classes are 38. As a common setting, these images will be firstly resized as $48\times 42$ and some sample images from the first two classes are illustrated in Figs. \ref{Fig:sampleimage_face}.
 \begin{figure}
  \centering
   \subfigure[]{\includegraphics[width=0.24\textwidth]{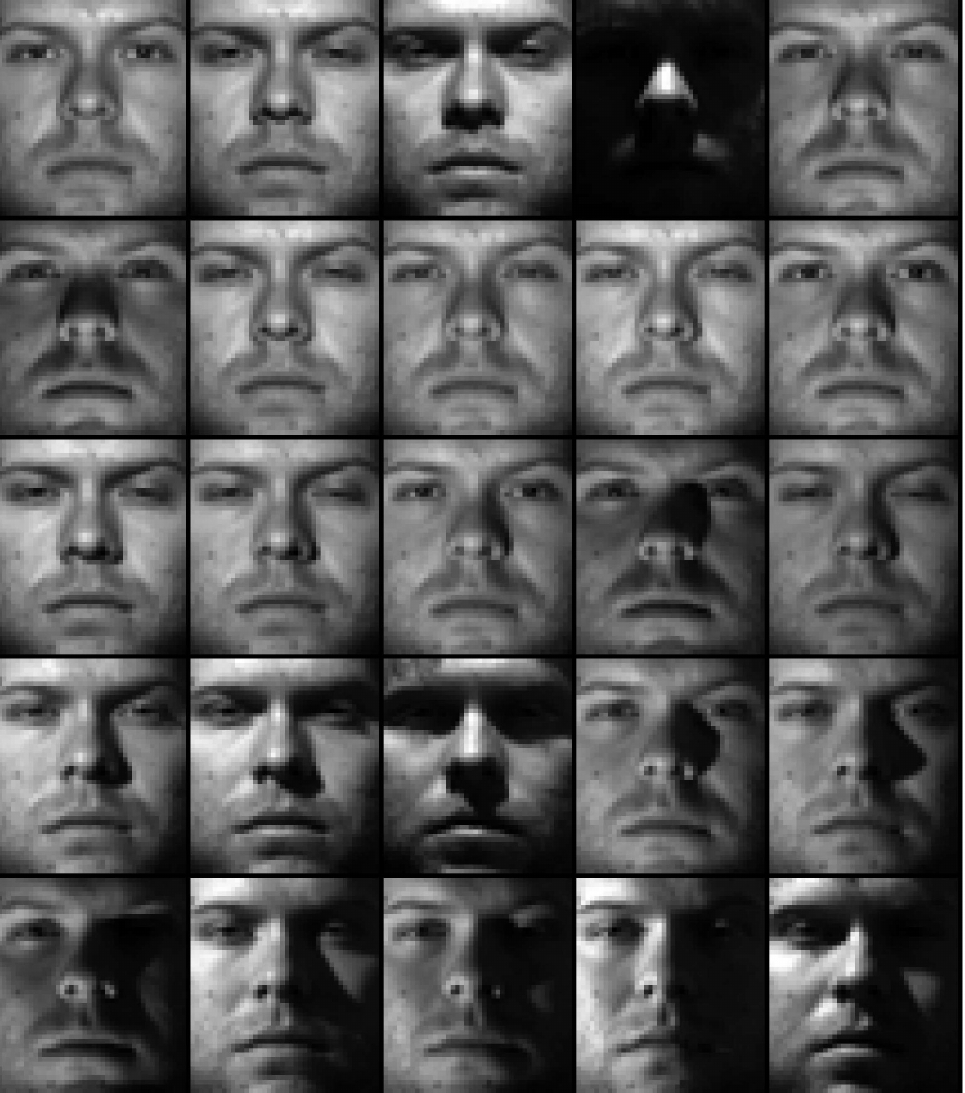}\label{Fig:EYaleB_Image1}}
  \subfigure[]{\includegraphics[width=0.24\textwidth]{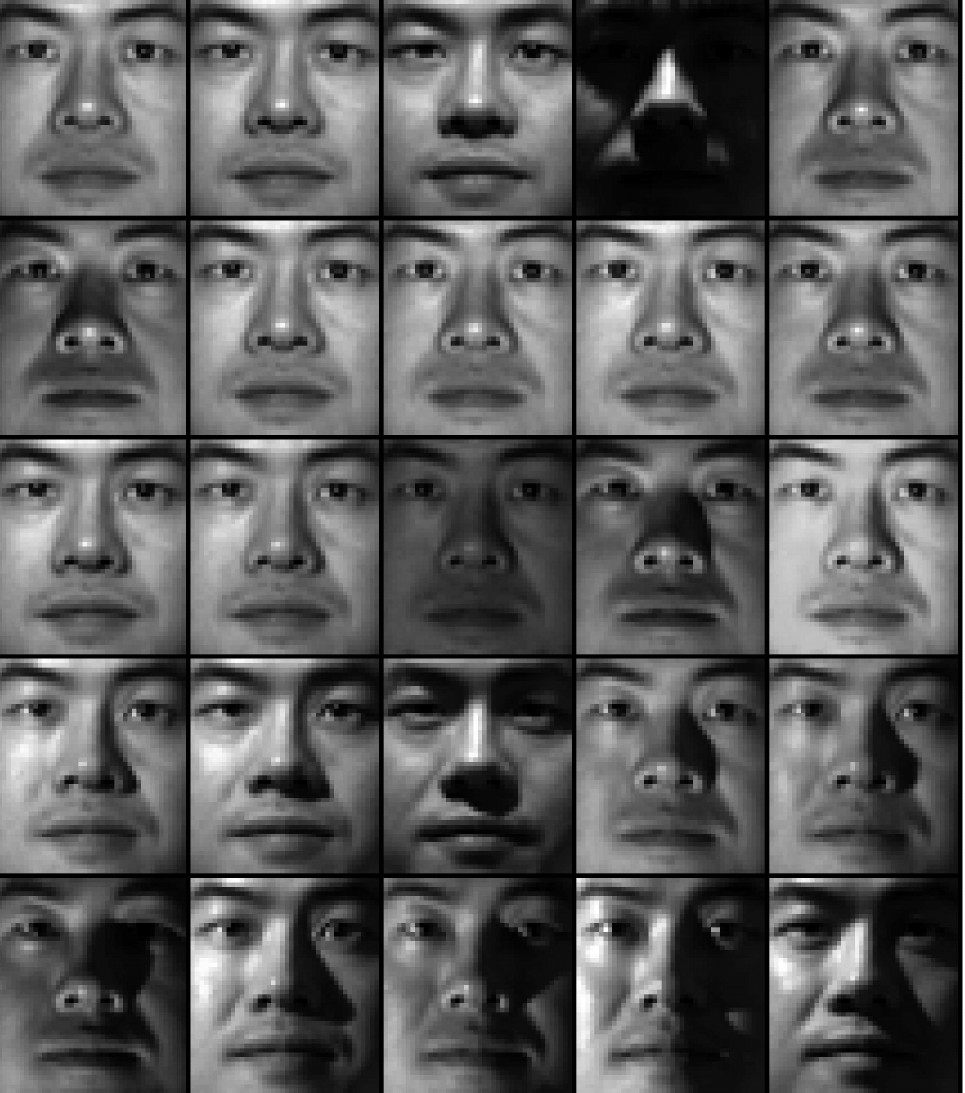}\label{Fig:EYaleB_Image2}}
  \caption{Sample images from two classes in EYB.}\label{Fig:sampleimage_face}
\end{figure}
\par For this database, we will consider to clustering in a $p=N$ dimensional feature domain with the choice of $K\in\{3,5,10\}$. Similar to the previous setting, 500, 200 and 100 independent experiments will be carried out for different choices of $K$. The parameters in SSC and LRR will be set as 20 and 1.8, respectively. The comparison results are summarized in Table. \ref{Table:FaceNdimensionresult}.
\begin{table}
\caption{Clustering Error (\%) on E-YaleB Dataset in the $N$-Dimensional Subspace\label{Table:FaceNdimensionresult}}\centering
\begin{tabular}{|c|c|c|c|c|c|c|c|}
\hline
K& Alg & LRT* & LS3C&NLS3C & LRR & SSC & DTL-FSSC\tabularnewline
\hline
\hline
\multirow{2}{*}{3} & mean & 5.83 &24.78&37.34 & 15.72 & 17.25& \textbf{1.93} \tabularnewline
\cline{2-8}
 & med. & 3.39 &24.48& 42.71 & 13.02 & 14.06 & \textbf{1.04}  \tabularnewline
\hline
\multirow{2}{*}{5} & mean & 3.69 &31.80& 51.14 & 15.33 & 27.47 & \textbf{3.22} \tabularnewline
\cline{2-8}
 & med. & 3.44 &31.25& 50.94 & 13.75& 24.84 & \textbf{2.19}  \tabularnewline
\hline
\multirow{2}{*}{10} & mean & 5.87 &35.89& 53.14& 17.03 &38.94 & \textbf{5.02}  \tabularnewline
\cline{2-8}
 & med. & 5.16 &35.08 &53.13& 16.56 &40.31 & \textbf{2.81} \tabularnewline
\hline
\end{tabular}
\end{table}
\par On this database, we can conclude from the results that both DTL-FSSC and LRT* have significantly improved the clustering performances than other comparison methods, and DTL-FSSC can still achieve the lower error rates than LRT*. Specifically, DTL-FSSC decreases the total clustering errors by almost $28.14\%$ than SSC.

\section{Conclusion}
In this paper, we develop a general DTL-FSSC framework of discriminative subspace clustering as well as dimensionality reduction. The novelty and innovation of the presented framework come from the following two folds. On one hand, DTL-FSSC enables us to learn a linear discriminative transformation and subspace clustering in a united framework, in which a fuzzy label matrix is specifically involved to enhance the discrimination and robustness. On the other hand, extensive experimental results show that DTL-FSSC can bring about significant improvements on clustering performances of three applications than other compared state-of-the-art frameworks, which will pay the price of the acceptable computational complexity.
\par To conclude this paper, we will present some promising and important issues as our ongoing and future researches. Firstly, the task of subspace clustering is essentially characterizing data distribution $p(\mathbf{X})$ with a suitable model. In the current frameworks, an approximation strategies indicated in \eqref{Equ:LowboundData} is exploited, where a shallow self-expressiveness linear synthesis model with a single regularized latent layer is taken into account to capture the subspace structures. With the prevalent strategies of deep learning \cite{LeCun2015}, a hierarchical structured model with multiple hidden layers will benefit from its more powerful capacity of capturing some high-order relationships among data. Therefore, exploiting some parametric deep generative models such as variational autoencoder \cite{Kingma2014} or generative adversarial networks \cite{Goodfellow2014}, to straightforward handle $p(\mathbf{X})$ will be preferred to receive a potentially better clustering performance. Another possible direction focuses on the transformation operator $\mathcal{A}:\mathcal{X}\rightarrow \mathcal{F}$ which can be also designed as a nonlinear deep neural network instead of the linear operator to enhance its capacity.
\label{sec:Conclusion}

\bibliographystyle{IEEEtran}
\bibliography{IEEEabrv,egbib}

\end{document}